%% file: main.tex
\documentclass[sigplan,screen]{acmart}
\usepackage{subfig}
\usepackage{amsmath}
\usepackage{graphicx}
\usepackage{mathrsfs}
\usepackage{amsfonts}
\usepackage{float}
\usepackage{multirow}

\usepackage{amssymb}
\usepackage{booktabs}
\usepackage{amsthm}
\usepackage{xspace}
\usepackage{booktabs}
\usepackage{color}

\newtheorem{assumption}{Assumption}[section]

\usepackage{tcolorbox}
\usepackage{listings}
\usepackage{comment}
\newtcolorbox{mybox}{
    colback=gray!5, 
    colframe=black, 
    arc=1mm, 
    boxrule=1pt, 
    left=1mm, 
    right=1mm, 
    top=0.5mm, 
    bottom=0.5mm 
}

\usepackage{pifont}

\AtBeginDocument{%
  }

\setcopyright{acmlicensed}
\copyrightyear{2018}
\acmYear{2018}
\acmDOI{XXXXXXX.XXXXXXX}
\acmConference[Conference acronym 'XX]{Make sure to enter the correct
  conference title from your rights confirmation email}{June 03--05,
  2018}{Woodstock, NY}

\acmISBN{978-1-4503-XXXX-X/2018/06}

\begin{document}

\title{Towards Robust Personalized Federated Learning: Vulnerability Assessment and Defense Co-Design}


\author{Mingyuan Fan}
\affiliation{%
  \institution{East China Normal University}
  \city{Shanghai}
  \country{China}}
\email{fmy2660966@gmail.com}

\author{Cen Chen}
\affiliation{%
  \institution{East China Normal University}
  \city{Shanghai}
  \country{China}
}
\email{cenchen@dase.ecnu.edu.cn}






\renewcommand{\shortauthors}{Trovato et al.}

\begin{abstract}
    The proliferation of IoT devices has fueled distributed edge systems to collect vast amounts of sensitive data, creating fertile ground for on-device machine learning applications.
    While federated learning (FL) mitigates privacy concerns by exchanging model parameters instead of raw data, we identify a critical blind spot in current research.
    We examine the most commonly used personalized federated learning (PFL) methods, which allow clients to maintain private, personalized models to address data heterogeneity across clients.
    Through systematic analysis, we reveal that PFL methods exhibit heightened vulnerability to transfer-based adversarial attacks compared to centralized learning paradigms.
    Wherein, malicious clients can exploit local model knowledge to craft adversarial examples that can compromise peer clients' personalized models.
    We establish this vulnerability through both theoretical analysis and empirical evaluation across multiple benchmark datasets, demonstrating significant accuracy drops across various PFL methods.
    To address this challenge, we propose a defense framework combining stochastic input noise, input-scaled trace regularization, and parameter sensitivity maximization to improve FL's robustness.
    Our findings establish the first systematic study of adversarial threats in PFL systems, providing both diagnostic tools and practical countermeasures.
\end{abstract}



\received{20 February 2007}
\received[revised]{12 March 2009}
\received[accepted]{5 June 2009}

\maketitle

\section{Introduction}
\label{sec:intro}
\input{tex/intro}

\section{Background}
\label{sec:background}

\input{tex/background}

\section{Adversarial Attack in PFL}
\label{sec:eval}
\input{tex/eval}

\section{What makes adversarial attack so effective in PFL?}
\input{tex/theory}

\section{How to mitigate adversarial attack in PFL?}
\input{tex/approach}

\section{Conclusion}
\label{sec:conclusion}
\input{tex/conclusion}

\clearpage


\bibliographystyle{ACM-Reference-Format}
\bibliography{main}

\clearpage

\appendix
\input{tex/appendix}

\end{document}

%% file: tex/intro.tex
The explosive expansion of Internet of Things (IoT) has spawned distributed networks of edge devices that amass vast volumes of sensitive data~\cite{laghari2021review,khan2021federated,fan2021toward}, yet such data remains largely confined to local nodes due to growing privacy regulations~\cite{wang2022accelerating}.
While such data holds immense potential for on-device machine learning applications like health monitoring and intrusion detection, traditional centralized training paradigms requiring raw data centralization causes severe privacy compromises~\cite{khan2021federated,xu2022adaptive}.
This fundamental tension has driven the adoption of federated learning (FL) in edge computing systems~\cite{chen2022decentralized,zhang2021elf}, which enables collaborative model training through parameter exchange while maintaining data localization across distributed devices.

In canonical FL, clients (edge devices) jointly train a shared model through FedAvg’s synchronized client-server protocol~\cite{fedavg}: clients download the global model, refine it using local datasets, and upload parameter updates for server-side aggregation.
This one-model-fits-all approach, however, suffers significant performance degradation when confronted with the intrinsic data heterogeneity of real-world edge environments~\cite{liao2024parallelsfl,wang2022accelerating}.
To address this, recent efforts have shifted toward personalized federated learning (PFL)~\cite{FedRep,FedBN,Ditto}, which empowers edge devices to cultivate private, personalized models that align with their distinctive data characteristics, while preserving the collaborative benefits of federated knowledge exchange.

Beyond personalization efforts, a critical orthogonal research direction is to recognize FL's intrinsic security challenges.
In detail, the distributed nature of FL brings broad attack surfaces, as current systems lack mutual integrity verification mechanisms between servers and clients or among clients themselves, enabling potential adversarial behavior from any participant.
Existing studies have identified multiple attack vectors, including model poisoning~\cite{multikrum}, backdoor injection~\cite{DBA,pgd_backdoor}, and privacy breaches~\cite{dlg}.
These threats have spurred countermeasures like secure aggregation~\cite{CryptoNets,BatchCrypt} and differential privacy~\cite{niu2020billion} to fortify FL systems.
Systematic exploration of novel attack surfaces remains crucial for both diagnosing FL's vulnerabilities and designing trust-aware frameworks that balance performance with robustness.

\begin{figure*}[!h]
    \centering
    \includegraphics[width=0.8\linewidth]{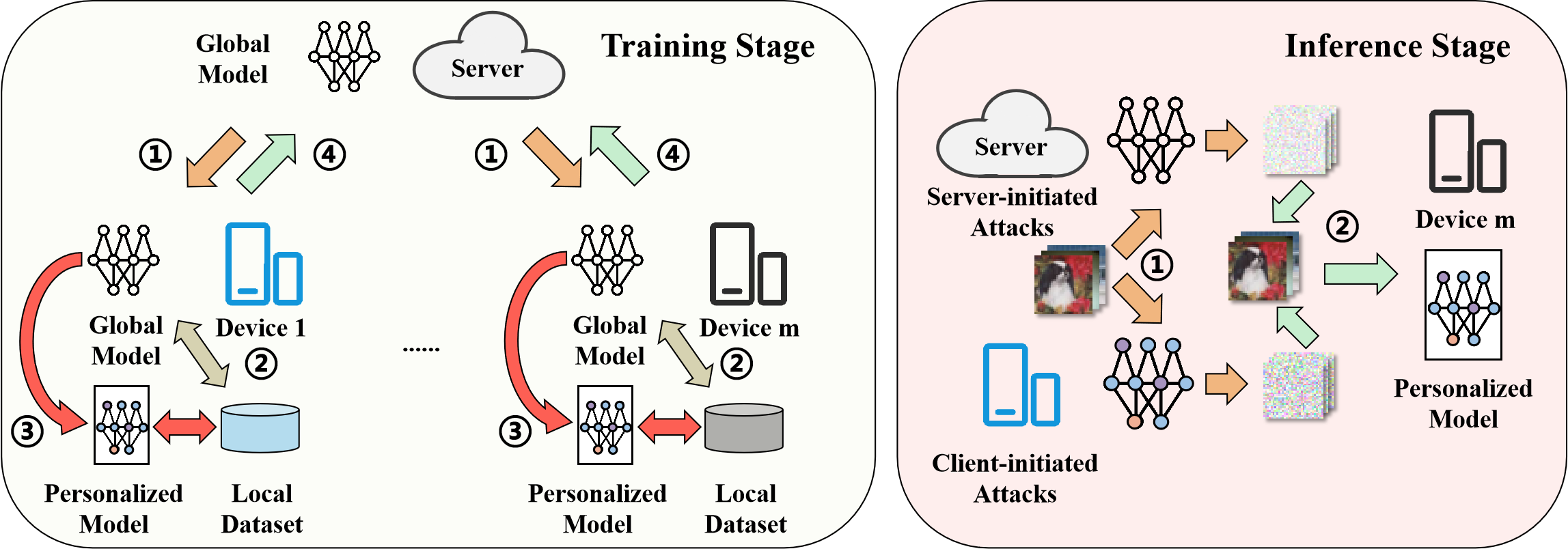}
    \caption{Overview. In the training phase, \ding{192} the server broadcasts the global model to clients.
    \ding{193} each client trains the global model using their local dataset.
    \ding{194} clients train their personalized models, which are regularized to ensure similarity to the global model through either soft or hard constraints.
    \ding{195} the updated global models are then uploaded to the server for aggregation to form a new global model.
    Notice that in some PFL methods, \ding{193} may be omitted, or the order of \ding{193} and \ding{194} may be swapped.
    Moreover, in canonical FL, personalized models (\ding{194}) do not exist.
    In the inference phase, we implement adversarial attacks to compromise clients' personalized models.
    There are two types of attacks: server-initiated and client-initiated.
    Both types follow a two-step process: \ding{192} the given data is fed into the model to generate adversarial perturbations, and \ding{193} these perturbations are added to the original data to produce adversarial samples to fool the target client's personalized model.
    The key difference between the two types of attacks lies in the model used to generate the perturbations: server-initiated attacks use the global model, while client-initiated attacks use the personalized model.}
    \label{fig_overview}
\end{figure*}

\textit{\textbf{Our first contribution is the unveiling of a critical yet unexplored attack surface: the vulnerability of PFL against adversarial attacks.}}
To the best of our knowledge, this constitutes the first investigation at the intersection of FL and adversarial attacks.
Adversarial attacks~\cite{fgsm,PGD} craft imperceptibly perturbed inputs, known as adversarial examples, to mislead target models into making incorrect classifications.
We study adversarial attacks in PFL and demonstrate that malicious PFL clients can weaponize local model knowledge to generate transferable adversarial examples capable of compromising peer clients' personalized models.
Our attack is similar to traditional transfer-based attacks~\cite{MI,multi_track} where an attacker trains a proxy model to produce adversarial examples against a black-box target model.
However, traditional transfer-based attacks primarily assume a pure black-box scenario, characterized by distinct model architectures and parameters between the proxy and target models. 
In contrast, our focus is on a distributed PFL scenario where clients' personalized models share identical model architectures but maintain different parameters, i.e., a gray-box scenario.

\textit{\textbf{Our second contribution is a comprehensive empirical analysis.}}
Through extensive experimentation, we reveal significant adversarial vulnerability in PFL implementations and provide useful insights for practitioners, particularly contrasting the robustness profiles of PFL models against centrally trained counterparts.
We identifiy several fundamental divergences, suggesting distinct vulnerability patterns that warrant further investigation.

\textit{\textbf{The third contribution is a principled theoretical framework.}}
We uncover five determinants of attack transferability in PFL systems, including model sensitivity, loss landscape geometry, distribution alignment, fitness difference, and parameter sensitivity.
Empirical validation confirms that PFL models exhibit worse values across these compared to centralized models, explaining their heightened susceptibility.

\textit{\textbf{The fourth contribution is the proposal of a robustness-enhanced PFL framework.}}
Leveraging our theoretical insight, we present three novel yet effective techniques, namely stochastic input noise augmentation, input-scaled trace regularization and parameter sensitivity minimization, to respectively to mitigate model sensitivity, loss landscape geometry, and parameter sensitivity.
These techniques are plug-and-play components, making our framework applicable to any PFL method.
Extensive experiments demonstrate our framework consistently enhances robustness against transfer-based attacks across diverse PFL implementations.

%% file: tex/background.tex
\begin{table}[]
    \caption{Notation table.}
    \label{tab:notation}
    \small
    \begin{tabular}{@{}cc@{}}
    \toprule
    Symbol & Definition \\ \midrule
    $x, y$        & Sample and its ground-truth label   \\
    $z$ & Feature embedding from feature extraction layers \\
    $p_i$ & Prediction probability of the $i$-th class \\
    $\mathcal{L}$  & Loss function   \\
    $D_i, P_{XY}^{(i)}$  & Client $i$'s local dataset and  data distribution   \\
    $\theta_g$  & Parameters of the global shared model   \\
    $\theta_i$  & Personalized model parameters of client $i$ \\
    $\eta$  & Learning rate for client models   \\
    $\delta$  & Adversarial perturbation   \\
    $\epsilon$  & Upper bound on the size of adversarial perturbation   \\
    $\mathcal{U}, S_t$  & Client set and selected clients in round $t$   \\
    $T$ & Total number of communication rounds \\
    $\tau$  & Local epochs   \\
    $\alpha$  & Step size for updating adversarial perturbation   \\
    $K$  & Number of attack iterations   \\ 
    $W$ & Weights of the classification head \\ \bottomrule
    \end{tabular}
\end{table}

We provide the necessary background here.
Table \ref{tab:notation} summarizes the mathematical notation for the reader's convenience.

\subsection{Personalized Federated Learning}

\paragraph{Federated learning}
Consider a FL system with $m$ clients and a central server.
Let $\mathcal{U}=\{ 1,2,\cdots,m \}$ denote the client set, where each client holds a private dataset $D_i$ drawn from its local data distribution $P_{XY}^i$, with $\mathcal{X}$ as the input space and $\mathcal{Y}$ as the $K$-class label space.
Given a loss function $\mathcal{L}: \mathcal{Y} \times \mathcal{Y} \rightarrow \mathbb{R}$, the canonical FL objective~\cite{fedavg} aims to learn a global model $F(x;\theta_g)$ with parameters $\theta_g$, formulated as:
\begin{equation}
\label{eq_vanilla_fl}
    \min_{\theta_g}  \sum_{i \in \mathcal{U}} \mathbb{E}_{(x,y) \sim D_i} \left[ \mathcal{L}(F(x;\theta_g), y) \right],
\end{equation}
where each term corresponds to a client's local empirical risk.
FedAvg~\cite{fedavg} optimizes Equation~\ref{eq_vanilla_fl} through alternating iterations of:
\begin{itemize}
    \item \textbf{Local training.} During communication round $t$, the server selects a subset $S_t \in \mathcal{U}$ of clients to participate. Each selected client $i \in S_t$ synchronizes its local model $\theta_g^i$ with the global model, i.e., $\theta_g^i \leftarrow \theta_g$, then performs $\tau$ epochs of SGD updates on $D_i$:
    \begin{equation}
    \theta_g^i = \theta_g^i - \eta \nabla_{\theta_g^i} \mathcal{L}(F(x;\theta_g^i), y), \ (x, y) \sim D_i,
    \end{equation}
    where $\eta$ is the learning rate.
    \item \textbf{Aggregation.} The server averages the locally trained models to form the new global model: $\theta_g = \frac{1}{|S_t|} \sum_{i \in S_t} \theta_g^i$.
\end{itemize}
FedAvg, however, suffers from performance degradation under data heterogeneity across clients~\cite{liao2024parallelsfl,wang2022accelerating,FedBN}, i.e., $P_{XY}^i \neq P_{XY}^j$ for $i \neq j$.
As shown in Figure \ref{fig_overview}, to mitigate this, PFL allows each client to maintain a private personalized model locally~\cite{FedProx,FedRep}.
For clarity, the local model refers to the global model downloaded for temporary local training (i.e., $\theta_g^i$), distinct from personalized models.
Let $\theta_i$ parameterize client $i$'s personalized model.

\paragraph{Full model-sharing methods}
Full model-sharing methods permit all elements of $\theta_i$ to differ from $\theta_g^i$, regularized by a coupling term $\mathcal{R}$:
\begin{equation}
\label{eq_full_pfl}
    \min_{\theta_i, i \in \mathcal{U}} \  \sum_{i \in \mathcal{U}} \ \mathbb{E}_{(x,y) \sim D_i } [ \mathcal{L}(F(x;\theta_i), y) + \mathcal{R}(\theta_i, \theta_g^i) ].
\end{equation}
In Equation \ref{eq_full_pfl}, $\mathcal{L}$ promotes local adaptation, while $\mathcal{R}$ preserves global knowledge through model alignment.
For example, FedProx~\cite{FedProx} implements $\mathcal{R}$ as $L_2$-norm penalty $||\theta_i - \theta_g^i||_2^2$ to constrain $\theta_i$ to remain close to $\theta_g^i$.
Ditto \cite{Ditto} first trains $\theta_g^i$ on $D_i$, followed by the training of $\theta_i$ in the same manner as FedProx.
SCAFFOLD~\cite{SCAFFOLD} introduces control variables to calibrate updates to $\theta_i$.

\paragraph{Partial model-sharing methods}
Partial model-sharing methods instead enforce hard parameter sharing for a subset of model parameters, while allowing others to specialize:
\begin{equation}
\label{eq_part_pfl}
\begin{split}
    &\min_{\theta_i, i \in \mathcal{U},} \ \sum_{i=1}^{m} \mathbb{E}_{(x,y) \sim D_i } \ [ \mathcal{L}(F(x;\theta_i), y) ] , \\
    &s.t., \, \theta_g^i [k] = \theta_i [k], \ k \in \Lambda, \ i \in \mathcal{U}.
\end{split}
\end{equation}
where $\Lambda$ indexes the shared parameters.
Partial model-sharing methods generally yield better performance by leveraging prior knowledge to identify which parameters are more likely to capture transferable representations for client sharing, while allowing others to vary freely for capturing client-specific patterns.
FedBN~\cite{FedBN} retains batch normalization layers locally while aggregating other layers through FedAvg.
Many later approaches adopt classification heads as personalization parameters.
For instance, FedRep~\cite{FedRep} and FedBABU~\cite{FedBABU} synchronize feature extractors but keep classification heads client-specific, with FedPAC~\cite{FedPAC} further improving these heads through weighted ensemble learning.
Moreover, FedAS~\cite{FedAS} aligns feature representations between local and personalized models before head training, while GPFL~\cite{GPFL} explicitly disentangles feature representations into distinct global and local representations further.
FedCAC~\cite{FedCAC} dynamically assigns $\Lambda$ via parameter sensitivity analysis, privatizing parameters that exhibit strong dataset-specific characteristics.

\subsection{Adversarial Attack}
\label{sec_2_2}

Modern edge computing systems increasingly deploy neural networks to enable intelligent functionalities across mobile applications~\cite{fan2021toward,shi2022audio,liao2024parallelsfl}.
However, these systems remain vulnerable to adversarial attacks, in which human-imperceptible perturbations are crafted to deceive models into making erroneous predictions.
Formally, given an input $x$ with ground-truth label $y$, the adversarial perturbation $\delta$ for $x$ can be obtained by solving:
\begin{equation}
\nonumber
\label{adv_formulation}
    \delta = \underset{\delta}{\arg \max} \ \mathcal{L}(F(x+\delta;\theta), y), \  s.t., ||\delta||_{\infty} \leq \epsilon.
\end{equation}
A higher loss value indicates stronger effectiveness in misleading the model, and $\epsilon$ bounds the magnitude of $\delta$ to ensure imperceptibility.

Early adversarial attacks focus on white-box scenarios assuming full model access.
PGD~\cite{PGD}, considered a benchmark method, solves Equation \ref{adv_formulation} through iterative projected gradient ascent with random initialization over $K$ iterations:
\begin{equation}
\label{adv_iter}
    \delta = \Pi_{[-\epsilon, \epsilon]} \left\{ \delta + \alpha \cdot \text{sign} \left( \nabla_{\delta} \mathcal{L}\left(F(x+\delta;\theta), y \right) \right\} \right),
\end{equation}
where $\Pi_{\epsilon}$ projects $\delta$ to the feasible $l_{\infty}$-norm ball and $\alpha$ is the step size.
Notably, PGD generalizes two fundamental attacks: basic iterative method~\cite{BIM} with deterministic zero initialization, and fast gradient sign method~\cite{fgsm} with $K=1$.
These methods require white-box access to the target model to derive input gradients, which is often impractical.

\paragraph{Transfer-based attacks}
To relax the white-box assumption, transfer-based attacks leverage adversarial examples generated from a proxy model to compromise black-box target models.
The attack effectiveness depends on the transferability of adversarial examples across distinct models.
The vanilla transfer-based attack applies PGD on the proxy model to produce adversarial examples, but their transferability is often limited, prompting the development of various transferability-enhancing techniques.
These techniques commonly involve modifications to PGD and can be categorized into three main types: optimization-based methods, input regularization methods, and model-based methods.
Optimization-based methods employ better optimization methods to escape suboptimal local maxima.
Examples include integrating momentum terms~\cite{MI} or calibrating gradient direction~\cite{VT} to stabilize update directions.
Input regularization methods introduce stochastic transformations to the input during perturbation generation (e.g., random resizing~\cite{DI}, scaling~\cite{NI_SI}, translation~\cite{TI}, and mask operation~\cite{maskblock}).
Model-based methods alter gradient propagation rules~\cite{sgm} or design proxy architectures that better approximate potential target models~\cite{ghost_net,multi_track}.

\begin{mybox}
\textbf{Takeaway for Section \ref{sec:eval}:}

\textit{1. PFL methods are highly susceptible to adversarial attacks, with significant accuracy drops observed across various datasets and models. The vulnerability is particularly pronounced in higher-performing models, highlighting an inherent accuracy-vulnerability trade-off.}

\textit{2. Partial-sharing PFL methods demonstrate better adversarial robustness compared to full-sharing methods at similar accuracy levels.}

\textit{3. In PFL, the transferability of adversarial examples is independent of the proxy model's accuracy, contrasting with centralized learning, where higher proxy accuracy correlates with better attack transferability.}

\textit{4. Server-initiated attacks are less effective against certain partial-sharing PFL methods because clients do not upload their classification heads.
}

\textit{5. Adversarial vulnerability in PFL increases with training rounds, perturbation budgets, and attack iterations. Advanced transfer-based attacks show only marginal benefits in PFL compared to centralized training.}

\end{mybox}

%% file: tex/eval.tex
\begin{figure*}[!h]
    \centering
    \subfloat[CIFAR-10]{\label{fig:a}\includegraphics[width=0.3\textwidth]{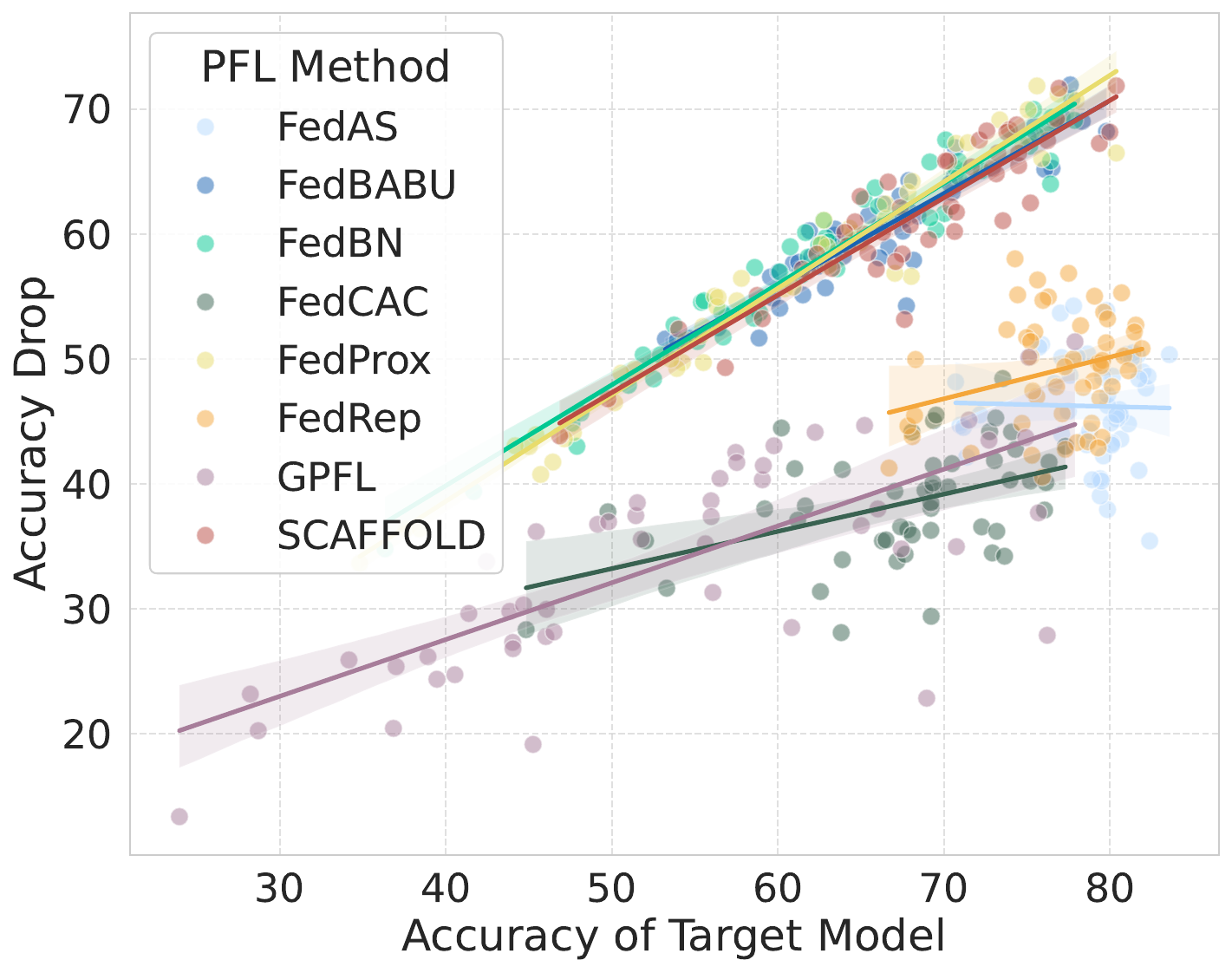}}
    \subfloat[CIFAR-100]{\label{fig:a}\includegraphics[width=0.3\textwidth]{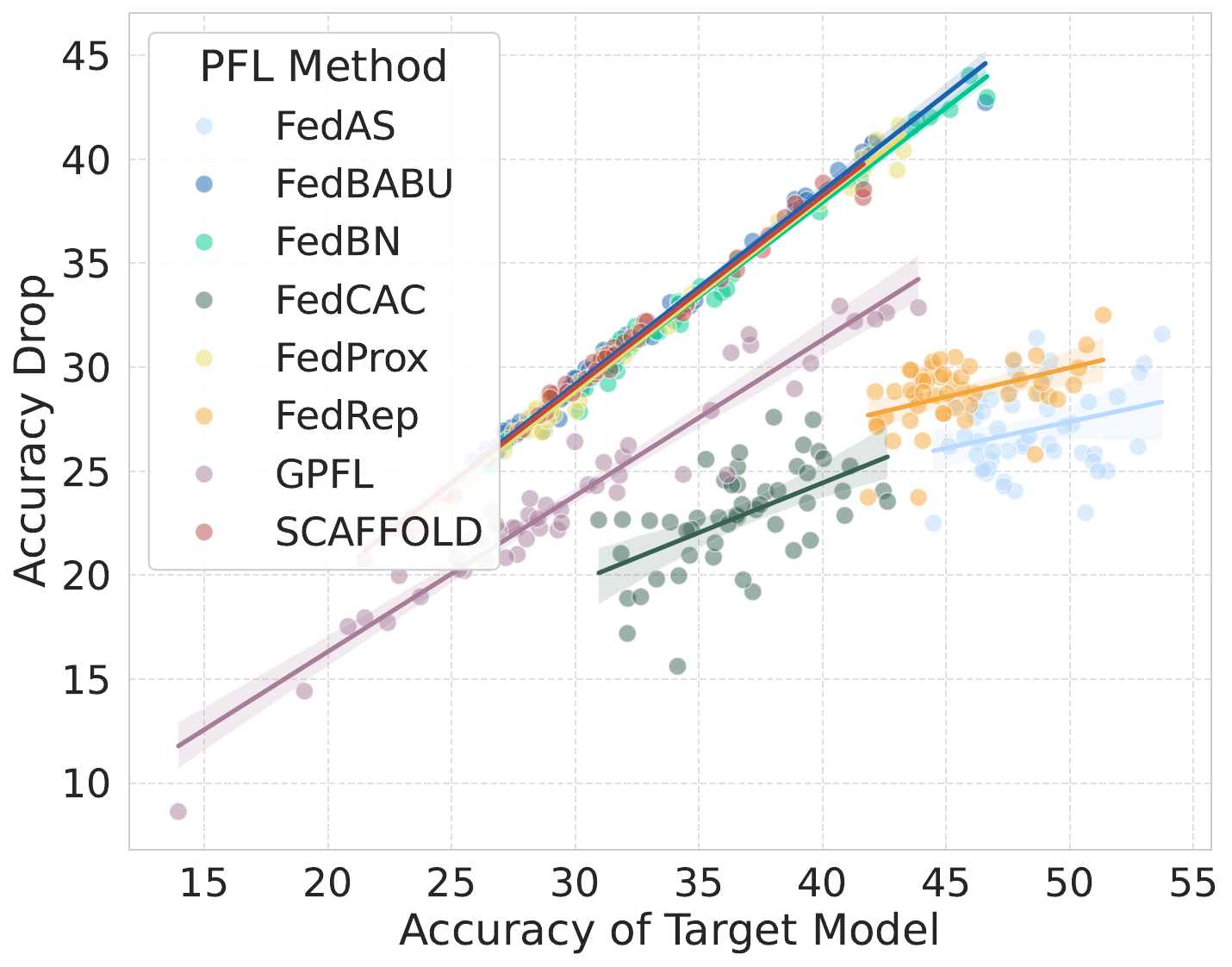}}
    \subfloat[GTSRB]{\label{fig:a}\includegraphics[width=0.3\textwidth]{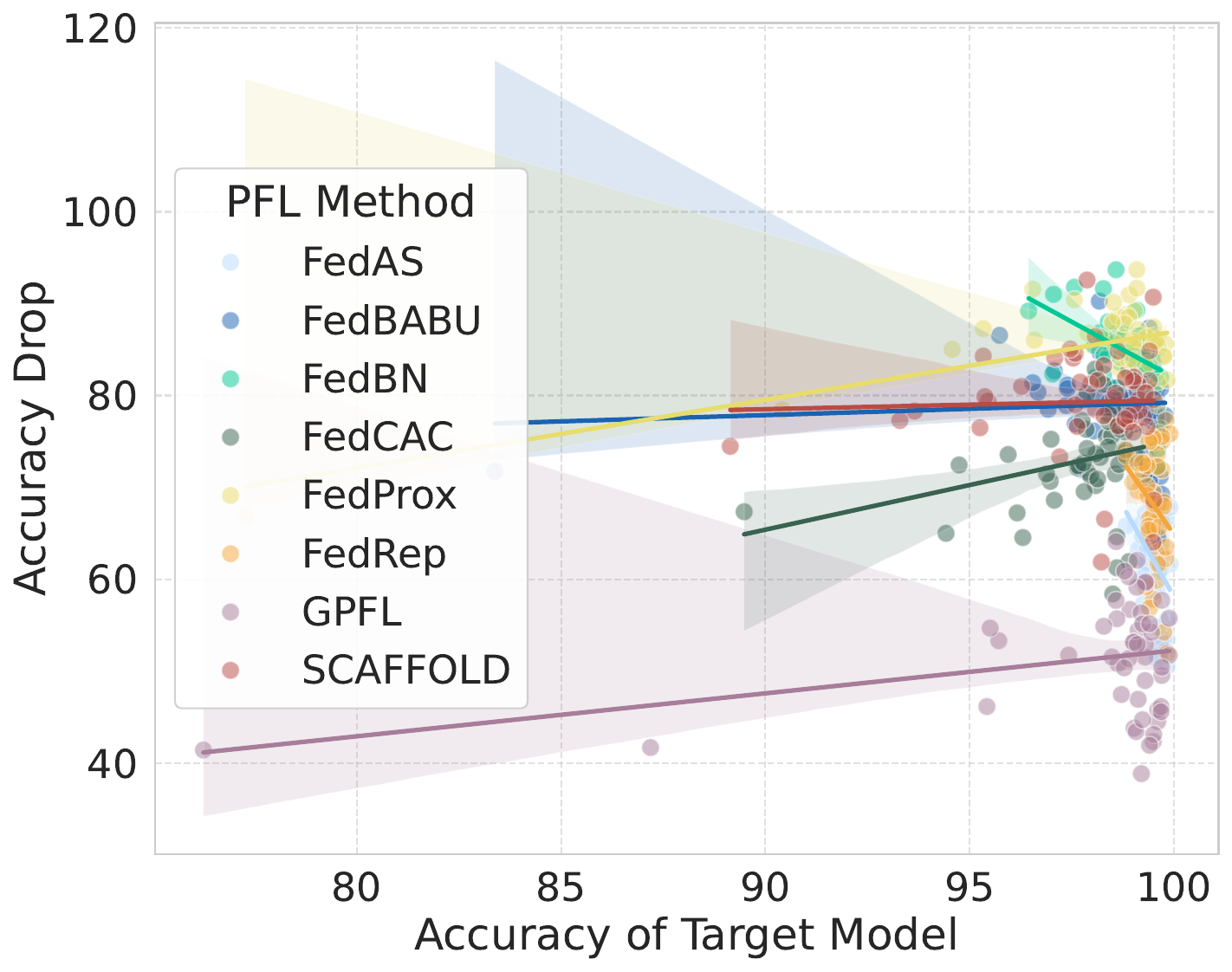}}
    \caption{Correlation between target model accuracy (\%) and AD (\%) in three datasets.}
    \label{sec3_fig1}
\end{figure*}

\begin{figure*}[!h]
    \centering
    \subfloat[CIFAR-10]{\label{fig:a}\includegraphics[width=0.3\textwidth]{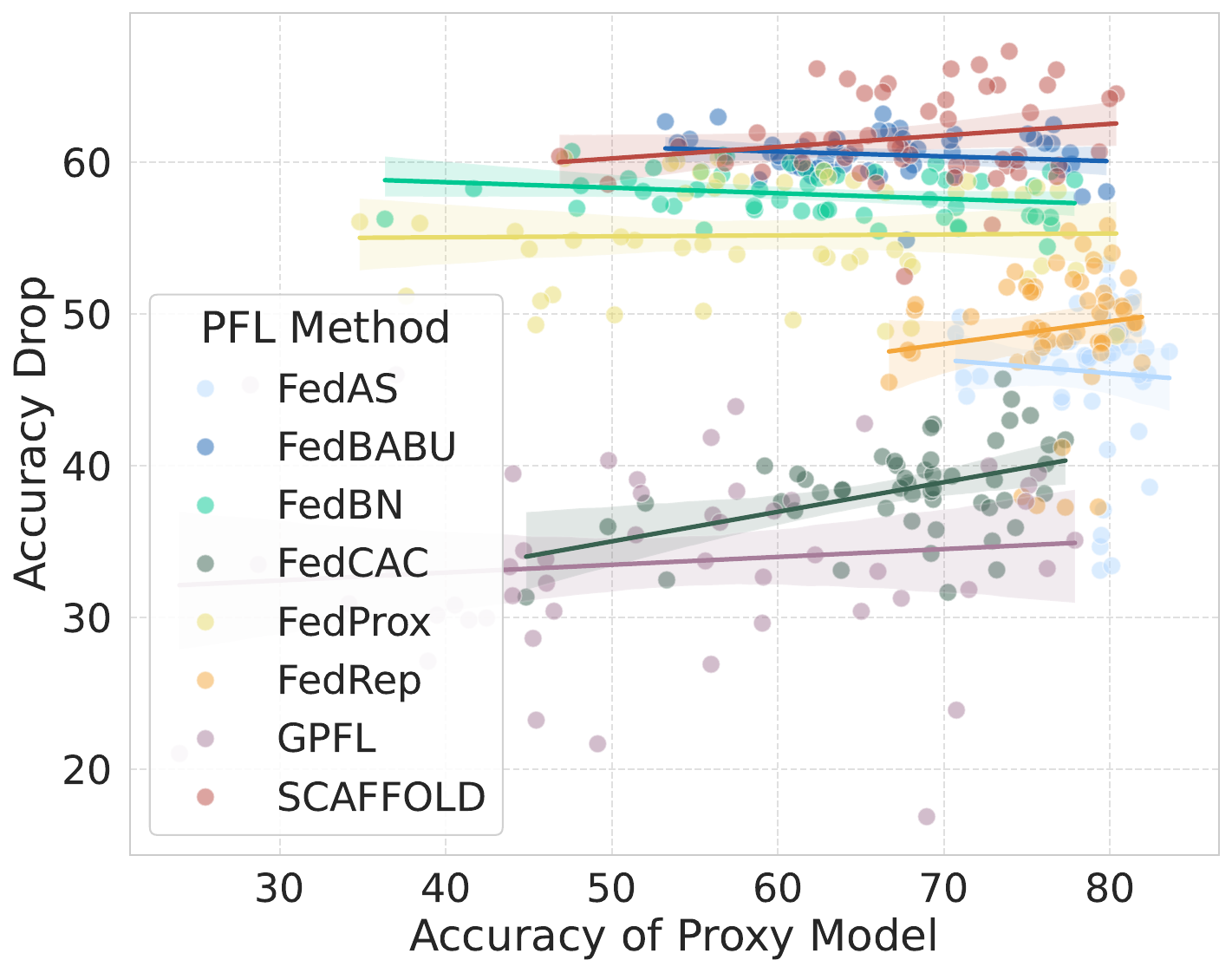}}
    \subfloat[CIFAR-100]{\label{fig:a}\includegraphics[width=0.3\textwidth]{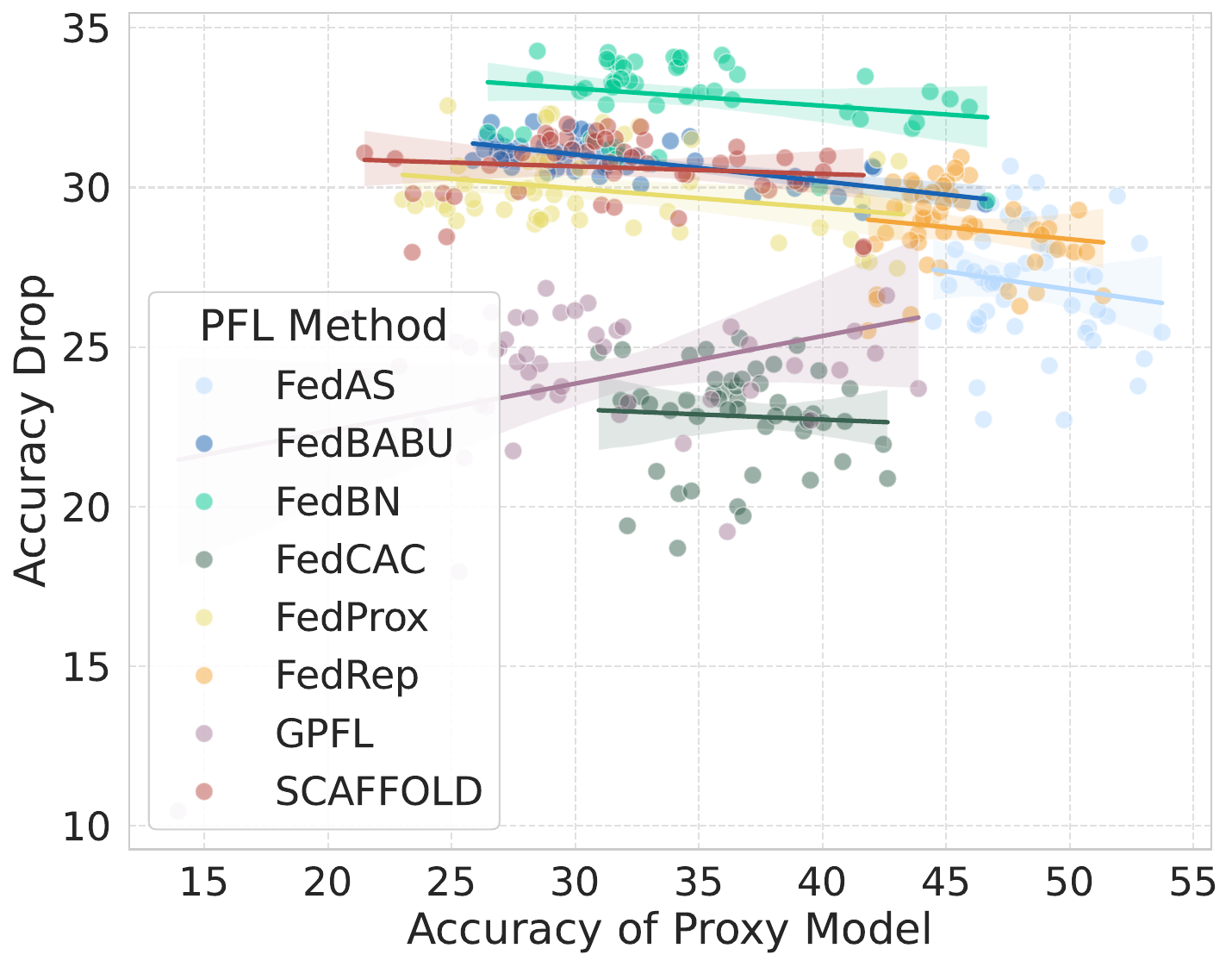}}
    \subfloat[GTSRB]{\label{fig:a}\includegraphics[width=0.3\textwidth]{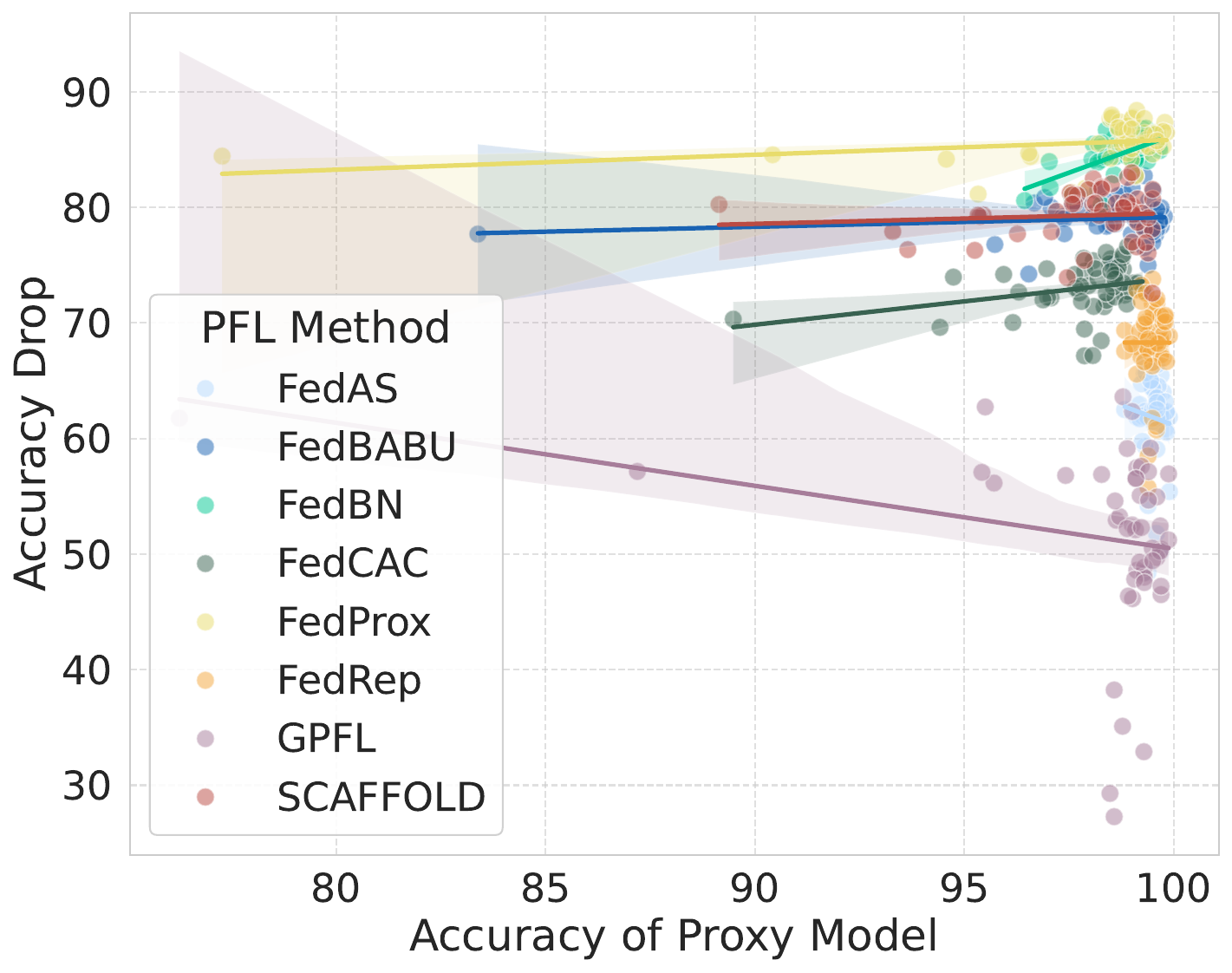}}
    \caption{Correlation between proxy model accuracy (\%) and AD (\%) in three datasets.}
    \label{sec3_fig2}
\end{figure*}

\begin{table*}[]
\caption{The effectiveness (average AD) of server-initiated attacks in three datasets.}
\label{sec3_tab1}
\centering
\small
\begin{tabular}{@{}ccccccccc@{}}
\toprule
Dataset   & FedProx & SCAFFOLD & FedBN & FedRep & FedBABU & GPFL & FedAS & FedCAC \\ \midrule
CIFAR-10  & 63.44   & 54.76    & 67.71 & 3.90   & 4.51    & 3.84 & 50.49 & 38.48  \\
CI-FAR100 & 33.54   & 33.33    & 30.09 & 3.22   & 1.87    & 3.98 & 30.40 & 25.79  \\
GTSRB     & 81.87   & 78.01    & 83.74 & 4.81   & 2.40    & 2.34 & 63.75 & 74.89  \\ \bottomrule
\end{tabular}
\end{table*}

\begin{figure}[!h]
\centering
\subfloat{\label{fig:a }\includegraphics[width=0.23\textwidth]{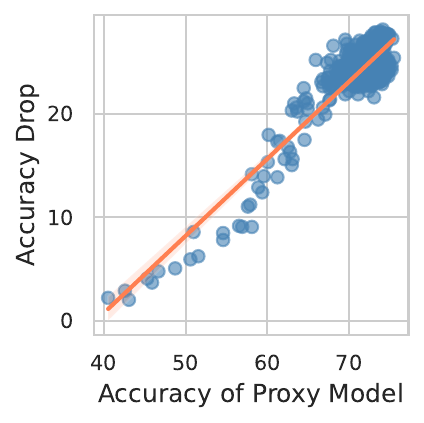}}
\subfloat{\label{fig:b }\includegraphics[width=0.23\textwidth]{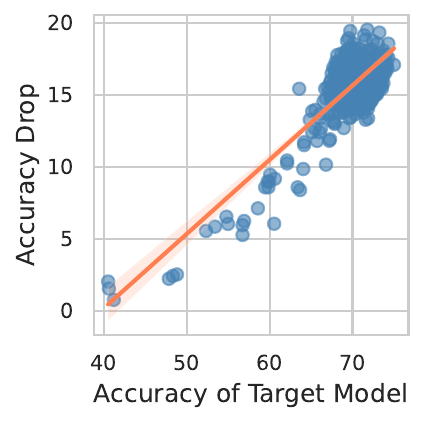}}
\caption{Centralized learning comparison. Left: AD versus proxy model accuracy with fixed target model. Right: AD versus target model accuracy with fixed proxy model.}
\label{sec3_fig3}
\end{figure}

\begin{figure*}[!h]
    \centering
    \includegraphics[width=0.8\linewidth]{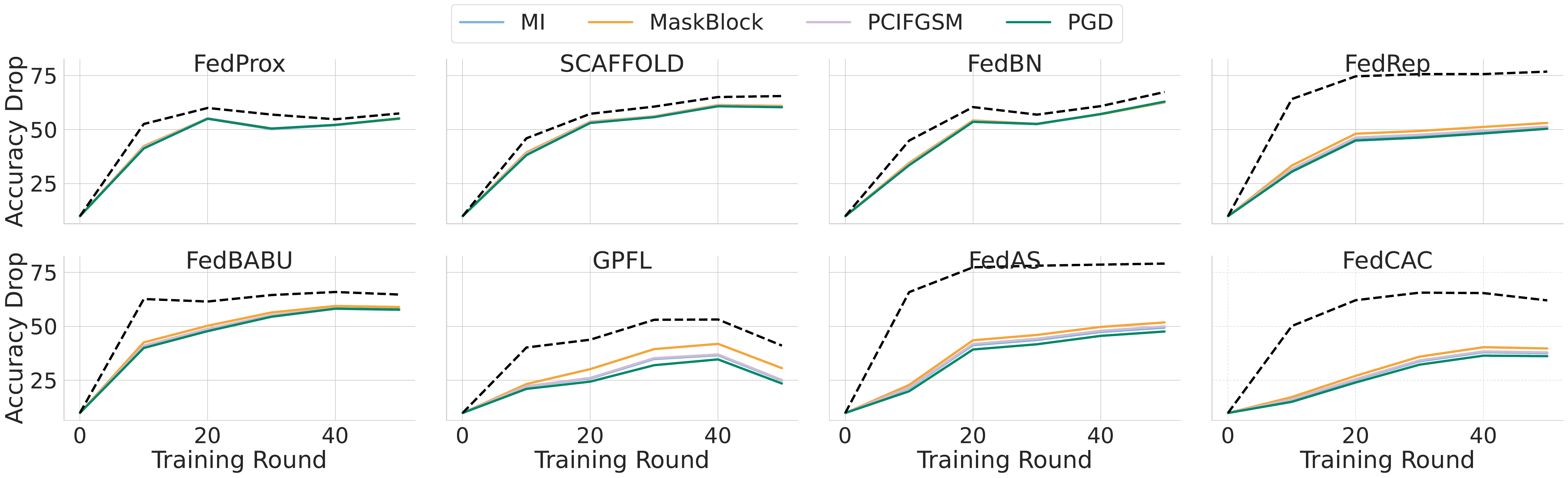}
    \caption{The average AD with different transfer-based attacks across varying training rounds. The black dotted lines indicate the model's average accuracy (shared on the y-axis).}
    \label{sec3_fig4}
\end{figure*}

\begin{figure*}[!h]
    \centering
    \includegraphics[width=0.8\linewidth]{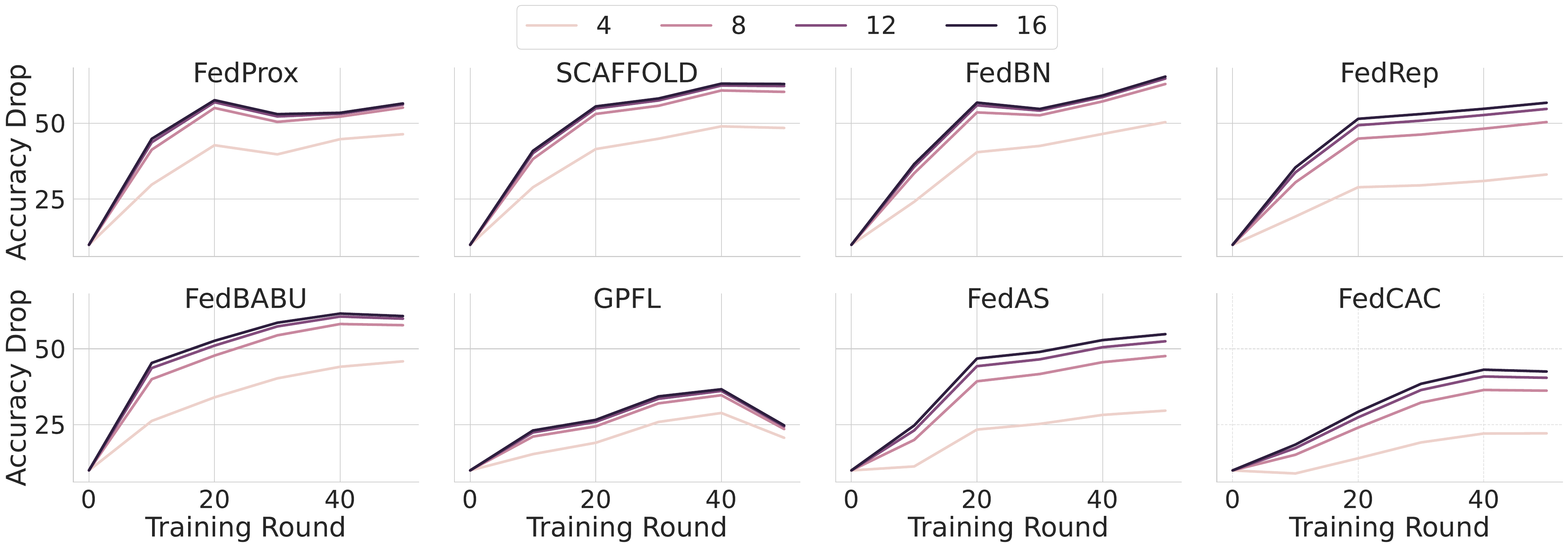}
    \caption{Impact of perturbation budget $\epsilon$ on average AD across training rounds using PGD.}
    \label{sec3_fig5}
\end{figure*}

\begin{figure*}[!h]
    \centering
    \includegraphics[width=0.8\linewidth]{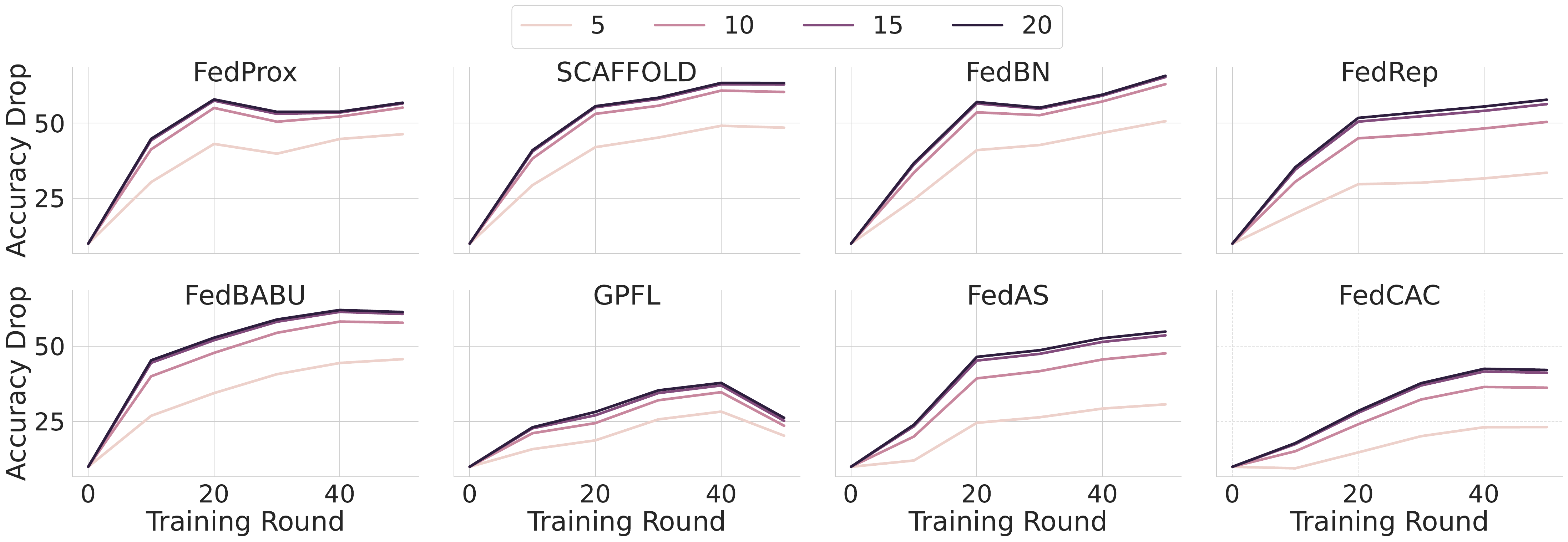}
    \caption{Impact of attack iterations on average AD across training rounds using PGD.}
    \label{sec3_fig6}
\end{figure*}

\begin{figure*}[!h]
    \centering
    \subfloat[Attack Method]{\label{fig:a}\includegraphics[width=0.26\textwidth]{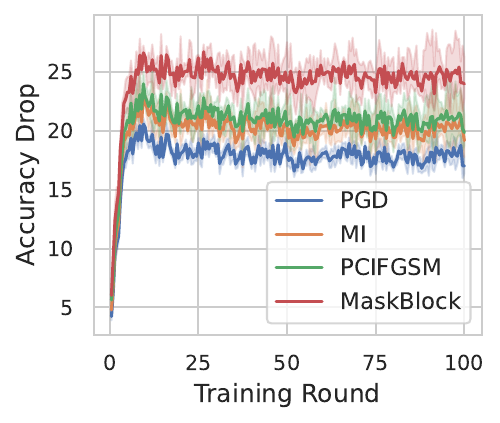}}
    \subfloat[Perturbation Budget]{\label{fig:b}\includegraphics[width=0.26\textwidth]{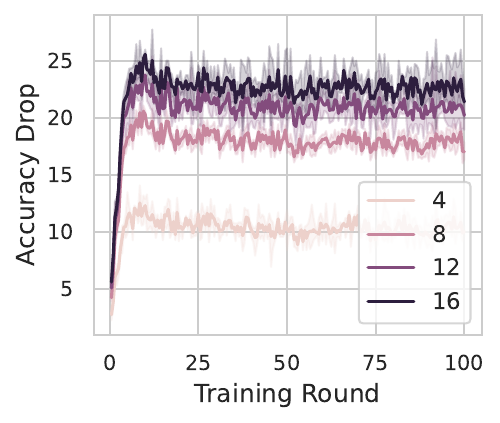}}
    \subfloat[Attack iteration]{\label{fig:b}\includegraphics[width=0.26\textwidth]{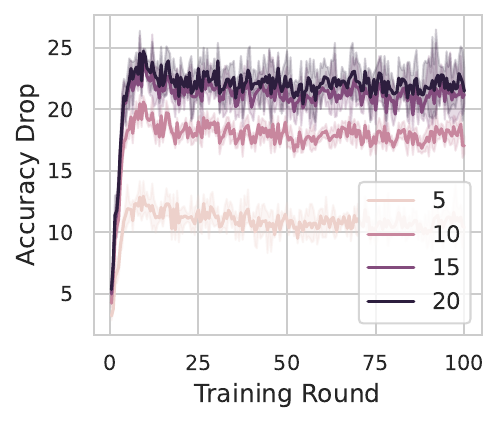}}
    \caption{Centralized training attack dynamics over different attack methods, perturbation budgets, and attack iterations. X-axis converted via iteration-to-round mapping (iterations / local iterations per round).}
    \label{sec3_fig7}
\end{figure*}


\subsection{Attack Scenario}
\label{sec_2_3}

We investigate a scenario in which multiple clients employ PFL methods to train their personalized models.
Our attack occurs \textit{during inference} after PFL training, operating under the assumption that the attacker can directly interact with the target victim model (e.g., via black-box API queries or physical device access) but cannot access its internal parameters or training data.
For instance, consider a distributed IoT surveillance network where edge devices (security cameras) train personalized intrusion detection models using local visual data.
An attacker could exploit one device's model to craft adversarial examples to undermine models deployed on other devices.
Moreover, we highlight that both the server and clients faithfully follow PFL training protocols, i.e., without malicious training-time behavior.
As shown in Figure \ref{fig_overview}, in this post-training attack regime, we identify two distinct attacks:
\begin{itemize}
    \item \textbf{Client-initiated attacks:} Malicious clients craft adversarial examples using their local personalized models to produce adversarial examples against other clients' personalized models.
    \item \textbf{Server-initiated attacks:} The server generates adversarial examples through the global model to target clients' personalized models.
\end{itemize}

\paragraph{Evaluation criteria}
Let $\delta=\mathcal{A}(x, y, \theta)$ denote the perturbation generated by attack method $\mathcal{A}$ (e.g., PGD) for input $(x,y)$ and model parameters $\theta$.
We define accuracy drop$_{i \rightarrow j}$ (AD$_{i \rightarrow j}$) to quantify the attack effectiveness from client $i$ to client $j$:
\begin{equation}
\nonumber
    \frac{{\sum_{(x_k, y_k) \in D_j}} [\mathbb{I}(F(x_k;\theta_j)=y_k) - \mathbb{I}(F(x_k+\delta_k;\theta_j)=y_k)]}{|D_j|},
\end{equation}
where $\delta_k=\mathcal{A}(x_k, y_k, \theta_i)$ and $\mathbb{I}$ represents indicator function.
This metric similarly applies to server-initiated attacks by substituting $\theta_i$ with $\theta_g$.
This metric is evaluated in the victim's test set $D_j$ to avoid artificial inflation from attacker-chosen inputs that might lie outside the model's operational domain.
The \textit{system vulnerability} of a PFL method is quantified through mean cross-client AD:
$$\frac{1}{(|\mathcal{U}|-1)^2} \sum_{i,j \in \mathcal{U}, i \neq j} AD_{i \rightarrow j}.$$
We explicitly exclude $i=j$ cases (white-box attacks) based on the practical constraint that attackers cannot access their target models' parameters.

\subsection{Setup}
\label{sec_3_2}

\textbf{FL configurations.}
We utilize ResNet-10 and three benchmark datasets, including CIFAR-10, CIFAR-100, and GTSRB, for evaluation.
The evaluation results for other models can be found in Appendix \ref{appendix_supp_exp}.
We simulate real-world statistical heterogeneity across clients using Dirichlet distribution $Dir(\beta)$ with $\beta=0.5$ by default, where each client $i$ receives $q_{c,i} \sim Dir(\beta)$ proportion of class $c$ samples \cite{GPFL}.
This induces both label distribution skew and quantity imbalance across clients, reflecting realistic non-IID scenarios.
Our FL system consists of 10 clients with 20\% participation rate ($|S_t| / |\mathcal{U}|=0.2$) over $T=50$ communication rounds (empirically verified for convergence).
All clients employ SGD with a batch size of 32, a learning rate $\eta$ of 0.1, and $\tau=5$ local epochs for both local and personalized models.

\noindent
\textbf{PFL methods.}
We evaluate eight state-of-the-art PFL methods, including FedProx \cite{FedProx}, SCAFFOLD \cite{SCAFFOLD}, FedBN \cite{FedBN}, FedRep \cite{FedRep}, FedBABU \cite{FedBABU}, GPFL \cite{GPFL}, FedAS \cite{FedAS}, and FedCAC \cite{FedCAC}.
The first two ones are full model sharing, while the remaining ones are partial model sharing.

\noindent
\textbf{Attacks.}
We harness PGD to generate adversarial examples on the global and personalized models, with $\epsilon=\frac{8}{255}$ and $K=10$ unless otherwise specified \cite{PGD}.
We will subsequently integrate state-of-the-art transferability-enhancing techniques, including MI \cite{MI}, PCIFGSM \cite{PCIFGSM}, and MaskBlock \cite{maskblock}, to study their impact on attack effectiveness.
MI stabilizes gradient updates via momentum term to escape local optima, while PCIFGSM adapts gradient computation rules.
MaskBlock applies dynamic input masking to regularize the attack process.
These methods represent three distinct categories of transferability-enhancing techniques: optimization-based (MI), model-based (PCIFGSM), and input regularization (MaskBlock).
All experiments execute 5 independent trials with different random seeds to ensure statistical significance. 






\subsection{Evaluation Result}

\paragraph{Target model accuracy and attack effectiveness.}
We first investigate the correlation between the performance of client-specific personalized models and their vulnerability to adversarial attacks.
Figure \ref{sec3_fig1} illustrates the per-client model accuracy plotted against these models' mean AD (computed by averaging AD$_{i \rightarrow j}$ over $i$) under client-initiated attacks.
Each PFL method has 50 data points, corresponding to five independent experimental trials.
We also incorporate linear regression fits to quantify statistical trends.
There are two key observations:
\begin{itemize}
    \item \textbf{Considerable adversarial vulnerability of PFL methods.}
    All personalized models exhibit substantial susceptibility to adversarial attacks, with $\geq$20\%, $\geq10$\% and $\geq$40\% AD on CIFAR-10, CIFAR-100, and GTSRB, respectively.
    Notably, the regression lines reveal a clear positive trend between target model accuracy and adversarial vulnerability, suggesting higher-performing models are more attack-prone, i.e., an accuracy-vulnerability trade-off in PFL.
    Exceptions arise in specific cases, such as FedAS in CIFAR-10 and all ones in GTSRB, where limited accuracy variance among client models likely obscures measurable trends.
    \item \textbf{Partial-sharing methods enjoy better accuracy-vulnerability trade-off.}
    Partial-sharing PFL methods (FedRep, FedAS, GPFL, FedCAC) demonstrate better adversarial robustness over full-sharing PFL methods at comparable accuracy levels.
    For instance, in CIFAR-10, FedRep achieves 70\% $\sim$ 80\% accuracy with AD around 50\%, outperforming full-sharing baselines like FedProx (40\%$\sim$80\% accuracy, 40\%$\sim$70\% AD).
\end{itemize}

\paragraph{Proxy model accuracy and attack effectiveness}
Figure~\ref{sec3_fig2} examines the relationship between proxy model accuracy and attack transferability, i.e., the average AD$_{i \rightarrow j}$ over target clients $j$ versus proxy model accuracy (evaluated on the target clients' test datasets).
We observe a counterintuitive phenomenon: adversarial transferability in PFL is largely independent of proxy model accuracy.
That is, AD remains consistent across proxy models with substantially different accuracy.
This phenomenon also suggests that any client can exploit its local personalized model against arbitrary target clients' personalized models, regardless of how well its own model performs.

Actually, one might intuitively expect that higher proxy accuracy correlates with better feature alignment between proxy and target models, thereby enhancing attack transferability.
To contextualize this phenomenon, we conduct a centralized learning\footnote{All client data is pooled together to train a model.} comparison where proxy and target models share identical architectures (ResNet-10) and training configurations (CIFAR-10, SGD with $\eta$=0.1), differing only in initialization:
\begin{itemize}
\item \textbf{Fixed target model}: We attack a trained target model (with 75\% accuracy) using proxy models from different training iterations.
\item \textbf{Fixed proxy model}: We use a trained proxy model to generate adversarial examples and evaluate their effectiveness in the target model across different training iterations.
\end{itemize}
Results in Figure~\ref{sec3_fig3} (over 5 trials) reveal a strong positive correlation between AD and both proxy/target model accuracy in centralized learning.
The stark contrast between centralized training and PFL stresses that proxy-accuracy-independent transferability constitutes a distinct vulnerability inherent to PFL.

\paragraph{Client-initiated attacks versus server-initiated attacks.}
Table \ref{sec3_tab1} reports the attack effectiveness of server-initiated attacks.
We see that server-initiated attacks can only achieve very low AD against certain parameter-sharing methods, specifically FedRep, FedBABU, and GPFL.
This is because, in these three PFL methods, clients do not upload their classification heads, rendering server's model ineffective.
For other methods, the performance difference between client-initiated attacks and server-initiated attacks is relatively small.
Given that server-initiated attacks can be effectively mitigated through partial model-sharing methods (by not uploading classification head) and then less threatening overall, subsequent analyses focus on client-initiated attacks.

\paragraph{Training dynamics and adversarial robustness.}
Figures~\ref{sec3_fig4}$\sim$\ref{sec3_fig6} analyze how adversarial robustness evolves across training rounds in PFL under various attack configurations, including different attack methods, perturbation budgets, and attack iterations.
For comparative analysis, we benchmark against centralized training scenarios (Figure \ref{sec3_fig7}), where two ResNet-10 models are trained independently, and the transferability of adversarial examples generated from one model to the other is assessed.
There are two observations.
\begin{itemize}
    \item \textbf{Monotonic increase in vulnerability.}
    Accuracy drop grows with training rounds, perturbation budgets, and attack iterations. This trend holds consistently for both PFL and centralized training.
    \item \textbf{Diminished returns from transfer-based attack methods against PFL methods.}
    While advanced transfer-based attacks, e.g., MaskBlock, improve cross-model transferability in centralized training by about 8\%, their benefits are less pronounced in PFL ($\leq$ 2\%).
\end{itemize}

%% file: tex/theory.tex
\paragraph{Theoretical Analysis of Attack Transferability in PFL}
Our empirical findings demonstrate that PFL models exhibit heighted vulnerability to transfer-based attacks compared to centralized models.
This intriguing phenomenon motivates our investigation into the intrinsic properties of PFL that causes this heighted vulnerability.
Consider two clients $i$ and $j$, with data distributions $P_{XY}^{(i)}, P_{XY}^{(j)}$, along with personalized model parameters $\theta_i$ and $\theta_j$.
We present a theoretical analysis of adversarial transferability between personalized models, where client $i$'s personalized model serves as the proxy and client $j$'s personalized model as the target.

For $(x,y)$ with perturbation $\delta$, we expand the loss function via first-order Taylor approximation:
\begin{equation}
\nonumber
\mathcal{L}(F(x+\delta, \theta_i), y) = \mathcal{L}(F(x, \theta_i), y) + \nabla_{x} \mathcal{L}(F(x, \theta_i), y)^T \delta.
\end{equation}
The perturbation maximizing loss is given by $\alpha \nabla_{x} \mathcal{L}(F(x, \theta_i), y)$, where $\alpha$ is a scaling constant ensuring validity of the linear approximation.
The adversarial effect of $\delta$ on client $j$'s model is then:
\begin{equation}
\begin{split}
\nonumber
    &\mathcal{L}(F(x+\delta, \theta_j), y) = \mathcal{L}(F(x, \theta_j), y) + \delta^T \nabla_{x} \mathcal{L}(F(x, \theta_i), y\\
    &= \mathcal{L}(F(x, \theta_j), y) + \alpha \nabla_{x} \mathcal{L}(F(x, \theta_j), y)^T \nabla_{x} \mathcal{L}(F(x, \theta_i), y).
\end{split}
\end{equation}
The above equation reveals that higher gradient inner product implies stronger attack effectiveness.
Notably, this conclusion can generalize to arbitrary $\delta$ magnitudes if models maintain gradient direction consistency in any point.

\begin{assumption}
\label{assumpt1}
Client $i$'s loss function $\mathcal{L}( F(x;\theta_i), y)$ is $K$-Lipschitz continuous over $\mathcal{X} \times \mathcal{Y}$ for some constant $K \geq 0$\footnote{Formally, we say that a function $f$ is $K$-Lipschitz continuous if $|f(x_1)-f(x_2)| \leq K ||x_1 - x_2||_2$ for all $x_1, x_2$ in its domain.}.
\end{assumption}

\begin{assumption}
\label{assumpt2}
Client $j$'s personalized model has converged on $P_{XY}^{(j)}$, i.e., $\mathbb{E}_{P_{XY}^{(j)}} [\nabla_{\theta_j} \mathcal{L}( F(x;\theta_j), y) ]=0.$
\end{assumption}

\begin{assumption}
\label{assumpt3}
Client $j$'s loss function $\mathcal{L}( F(x;\theta_j), y)$ is strongly convex in model parameters $\theta_j$.
\end{assumption}

\begin{theorem}
\label{thm1}
(See Appendix \ref{theory_proof} for Proof.)
Let $\lambda_{\text{min}}$ be the minimal eigenvalue of $\nabla_{\theta_j}^2 \mathcal{L}(F(x, \theta_j), y)$.
Use $\Delta$ to denote $\theta_i-\theta_j$.
The gradient inner product is bounded as follows:
\begin{equation}
\begin{split}
\nonumber
&\nabla_x \mathcal{L}(F(x;\theta_i), y)^\top \nabla_x \mathcal{L}(F(x;\theta_j), y) \\
&\geq \frac{1}{2} {\left\{||\nabla_x \mathcal{L}(F(x;\theta_i), y)||_2^2 + ||\nabla_x \mathcal{L}(F(x;\theta_j), y)||_2^2 \right\}} \\
&- C || \nabla_{x,\theta_j}^2 \mathcal{L}(F(x;\theta_j), y)||_2^2 + \mathcal{O}(||\Delta||_2), \\
\end{split}
\end{equation}
where $C = \frac{1}{\lambda_{\text{min}}} \{K \cdot  W(P_{XY}^{(i)}, P_{XY}^{(j)}) -\mathbb{E}_{P_{XY}^{(i)}} \left[ \mathcal{L}(F(x;\theta_i), y) \right] + \mathbb{E}_{P_{XY}^{(j)}} \left[ \mathcal{L}(F(x;\theta_j), y) \right] \}.$
\end{theorem}

We employ Wasserstein distance to quantify the distributional divergence between $P_{XY}^{(i)}$ and $P_{XY}^{(j)}$, i.e., $W(P_{XY}^{(i)}, P_{XY}^{(j)})$.
Under Assumptions \ref{assumpt1}$\sim$\ref{assumpt3}, the above theorem establishes a lower bound for gradient inner product.
Before analyzing Theorem~\ref{thm1}, we discuss its applicability in PFL context.
Assumption \ref{assumpt1} is generally valid in practice since unbounded Lipschitz constants would imply highly irregular or pathological loss landscapes.
Assumption \ref{assumpt2} is obvious.
Assumption \ref{assumpt3} holds true for linear models with cross-entropy loss.
In many PFL methods, gradient differences mainly originate from classification heads, and thus Assumption \ref{assumpt3} is also likely to hold, at least for those methods that personalize only the classification heads.
Furthermore, when the model converges, the loss landscape is often locally convex near the optimum.
In Theorem \ref{thm1}, the term $\mathcal{O}(||\Delta||_2)$ when the parameter differences between personalized models are small.
Since PFL methods typically introduce regularization terms on parameter differences, either soft (Equation \ref{eq_full_pfl}) or hard (Equation \ref{eq_part_pfl}), the parameter differences between personalized models are usually small.

Theorem~\ref{thm1} reveals several critical factors:
\begin{itemize}
    \item \textbf{Model sensitivity $||\nabla_x \mathcal{L}(F(x;\theta_i), y)||_2^2 + ||\nabla_x \mathcal{L}(F(x;\theta_j), y)||_2^2 $}: Higher input gradient norms indicate that the model is more sensitive to input perturbations, amplifying the potential for adversarial transferability.
    \item \textbf{Loss landscape geometry $\lambda_{\text{min}}$}: Intuitively, $\lambda_{\text{min}}$ characterizes the curvature of the loss landscape, specifically along the direction of least resistance. A substantial $\lambda_{\text{min}}$ implies that even the flattest direction exhibits significant curvature, forcing the loss to vary substantially along all directions. This geometric property enhances the target model's sensitivity to input perturbations: minor adversarial perturbations induce pronounced loss changes when $\lambda_{\text{min}}$ is large, thereby enhancing adversarial transferability.
    \item \textbf{Distribution alignment and model fitness difference $-\mathbb{E}_{P_{XY}^{(i)}} \left[ \mathcal{L}(F(x;\theta_i), y) \right] + \mathbb{E}_{P_{XY}^{(j)}} \left[ \mathcal{L}(F(x;\theta_j), y) \right] + K W(P_{XY}^{(i)}, P_{XY}^{(j)})$}: When the Wasserstein distance between their respective data distributions is small, the corresponding underlying feature distributions are more likely to overlap. Furthermore, if the models achieve comparable performance on their respective datasets (as reflected by their expected loss values), the features learned by the models awill exhibit greater similarity. Together, these factors lead to higher transferability.
    \item \textbf{Parameter sensitivity $|| \nabla_{x,\theta_j}^2 \mathcal{L}(F(x;\theta_j), y)||_2^2$}: Smaller values indicate input gradient invariance to parameter changes. Negligible parameter sensitivity (approaching zero) would decouple input gradient behavior from model parameterization. This implies that the regions of high input sensitivity (e.g., adversarial vulnerabilities) remain consistent across different parameter instantiations of the target model. Consequently, adversarial examples crafted on a proxy model by exploiting its input-sensitive regions (high gradient norms) will likely transfer effectively to the target model, as its sensitivity profile remains stable despite parameter changes.
\end{itemize}

\begin{figure}[!t]
    \centering
    \subfloat{\label{fig:b }\includegraphics[width=0.248\textwidth]{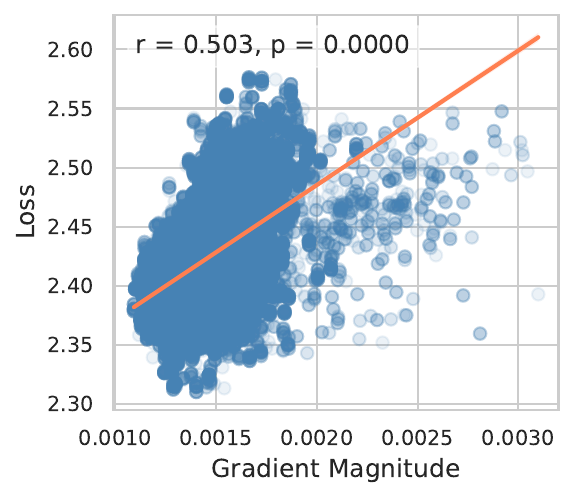}}
    \subfloat{\label{fig:b }\includegraphics[width=0.25\textwidth]{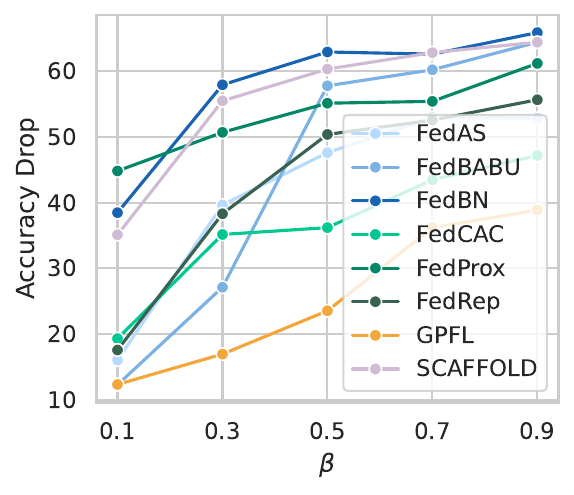}}
    \caption{\textbf{Left:} The sum of the input gradient norms on the proxy and target models, along with the loss values of the generated adversarial examples on the target model. \textbf{Right:} The performance of transfer-based attacks under varying degrees of data heterogeneity.}
    \label{sec4_fig1}
\end{figure}

\begin{table}[!t]
\small
\centering
\caption{Comparison of metrics between centralized training and PFL methods.}
\label{sec4_tab1}
\begin{tabular}{@{}cccc@{}}
\toprule
\begin{tabular}[c]{@{}c@{}}Training\\ Paradigm\end{tabular} & \begin{tabular}[c]{@{}c@{}}Model\\ Sensitivity\end{tabular} & \begin{tabular}[c]{@{}c@{}}Average\\ Eigenvalue\end{tabular} & \begin{tabular}[c]{@{}c@{}}Parameter\\ Sensitivity\end{tabular} \\ \midrule
Centralized Training          & 0.0078                                                      & 2.984$\times 10^{-6}$                                        & 0.426                                                           \\
PFL methods                                                         & 0.0142                                                      & 9.712$\times 10^{-6}$                                        & 0.545                                                           \\ \bottomrule
\end{tabular}
\end{table}

\paragraph{Empirical validations}
We empirically validate the above terms.
For model sensitivity, we assess the gradient norms of original samples on both the proxy and target models, followed by evaluating the effectiveness (loss magnitude) of adversarial examples on the target model.
Figure \ref{sec4_fig1} demonstrates a strong positive correlation between gradient norms and attack effectiveness across diverse PFL client models.
Next, we investigate the geometry of the loss landscape by evaluating the average eigenvalue of the Hessian matrix associated with the classification head.
The correlation coefficient between the average eigenvalue and AD is -0.176.
For distribution alignment, we vary $\beta$ from $0.1$ to $0.9$, where a larger $\beta$ corresponds to reduced data heterogeneity (a smaller Wasserstein distance).
As shown in Figure \ref{sec4_fig1}, increasing $\beta$ leads to a more pronounced degradation in the target model's accuracy, aligning with our theoretical analysis.
Furthermore, we evaluate the correlation coefficient between the performance gap of proxy and target models and AD.
The resulting correlation coefficient is -0.218.
This inverse relationship indicates that smaller fitness gaps between proxy and target models are associated with heightened vulnerability.
Finally, we analyze the relationship between parameter sensitivity and AD.
Our measurements reveal a moderate correlation coefficient of 0.109, corroborating our theoretical analysis that greater parameter sensitivity makes higher AD.

\paragraph{PFL versus centralized training}
Table \ref{sec4_tab1} presents a comparison of these metrics associated with models trained with centralized training and PFL methods.
Overall, the PFL paradigm exhibits higher model sensitivity, larger average Hessian eigenvalues, and increased parameter sensitivity compared to centralized training.
These quantitative differences shed light on the empirically observed vulnerability gap between the two paradigms.
These motivate us to propose a robustness-enhanced PFL framework aimed at mitigating these issues.
Notice that we do not consider distribution alignment and fitness difference here, as these factors are inherently fixed and cannot be altered within the given scenario.
Instead, our efforts concentrate on addressing the modifiable aspects of model behavior to improve robustness effectively.

%% file: tex/approach.tex

\subsection{Robustness-enhanced PFL Framework}

\paragraph{Overview}
Building upon our theoretical foundations, we introduce a robustness-enhanced PFL framework, which involves three key techniques, namely stochastic input noise augmentation, input-scaled trace regularization, and parameter sensitivity minimization.
These techniques are designed to regularize input gradients, $\lambda_{\text{min}}$, and parameter sensitivity, respectively.
Moreover, these techniques are plug-and-play and thus are compatible with arbitrary PFL methods, making our framework a universal robustness enhancement solution for both existing and emergent PFL methods.
Taking client $i$ as a representative example, its original personalized model optimization objective is formulated as:
\begin{equation}
\begin{split}
\min_{\theta_i} \ \mathbb{E}_{(x,y) \sim \mathcal{D}_i } \left[ \mathcal{L}(F(x;\theta_i), y) \right],
\end{split}
\end{equation}
where we intentionally omit the soft constraints from Equation \ref{eq_full_pfl} and hard constraints from Equation \ref{eq_part_pfl} for narrative clarity, focusing instead on the core robustness mechanisms.

\begin{table*}[!ht]
\small
\caption{The performance of the proposed robustness-enhanced framework in CIFAR-10.}
\label{sec5_tab1}
\begin{tabular}{@{}c|cccccccccccccccc@{}}
\toprule
PFL Method   & \multicolumn{2}{c}{FedAS} & \multicolumn{2}{c}{FedBABU} & \multicolumn{2}{c}{FedBN} & \multicolumn{2}{c}{FedCAC} & \multicolumn{2}{c}{FedProx} & \multicolumn{2}{c}{FedRep} & \multicolumn{2}{c}{GPFL} & \multicolumn{2}{c}{SCAFFOLD} \\ \midrule
Metric       & Acc         & AD          & Acc          & AD           & Acc         & AD          & Acc          & AD          & Acc          & AD           & Acc          & AD          & Acc         & AD         & Acc           & AD           \\ \midrule
Vanilla      & 79.05       & 47.61       & 66.72        & 57.79        & 67.39       & 62.94       & 62.00        & 36.19       & 57.46        & 55.14        & 76.83        & 50.38       & 51.08       & 23.52      & 65.51         & 60.34        \\
Ours          & 79.32       & 40.57       & 78.94        & 30.72        & 77.85       & 53.00       & 72.62        & 28.16       & 68.34        & 46.01        & 78.19        & 42.84       & 60.97       & 18.37      & 74.79         & 50.21        \\
Ours+AT & 69.36       & 8.46        & 70.49        & 8.42         & 66.90       & 17.12       & 62.47        & 11.03       & 53.27        & 17.02        & 68.19        & 11.50       & 48.26       & 14.19      & 59.41         & 18.11        \\
Ours+AT+RS    & 51.35       & 3.57        & 54.78        & 3.51         & 49.46       & 6.69        & 47.78        & 5.94        & 31.86        & 4.91         & 48.66        & 4.36        & 34.62       & 3.40       & 38.42         & 5.62         \\ \bottomrule
\end{tabular}
\end{table*}

\begin{table*}[!ht]
\small
\caption{The performance of the proposed robustness-enhanced framework in CIFAR-100.}
\label{sec5_tab2}
\begin{tabular}{@{}c|cccccccccccccccc@{}}
\toprule
PFL Method   & \multicolumn{2}{c}{FedAS} & \multicolumn{2}{c}{FedBABU} & \multicolumn{2}{c}{FedBN} & \multicolumn{2}{c}{FedCAC} & \multicolumn{2}{c}{FedProx} & \multicolumn{2}{c}{FedRep} & \multicolumn{2}{c}{GPFL} & \multicolumn{2}{c}{SCAFFOLD} \\ \midrule
Metric       & Acc         & AD          & Acc          & AD           & Acc         & AD          & Acc          & AD          & Acc          & AD           & Acc          & AD          & Acc         & AD         & Acc           & AD           \\ \midrule
Vanilla      & 48.43       & 26.98       & 31.84        & 30.88        & 34.43       & 32.86       & 36.66        & 22.83       & 31.15        & 29.90        & 45.51        & 28.72       & 30.09       & 23.89      & 31.73         & 30.62        \\
Ours          & 49.02       & 11.55       & 44.13        & 12.12        & 41.45       & 39.10       & 36.63        & 18.29       & 35.96        & 26.38        & 45.64        & 21.25       & 29.98       & 20.57      & 36.29         & 25.32        \\
Ours+AT & 37.73       & 6.87        & 38.04        & 6.57         & 33.78       & 15.38       & 31.30        & 8.94        & 28.64        & 14.36        & 37.95        & 9.90        & 21.51       & 8.94       & 29.64         & 13.09        \\
Ours+AT+RS    & 16.90       & 1.79        & 17.37        & 1.69         & 13.05       & 3.22        & 20.52        & 4.35        & 10.02        & 2.58         & 14.87        & 1.74        & 7.24        & 1.04       & 9.15          & 1.71         \\ \bottomrule
\end{tabular}
\end{table*}

\paragraph{Stochastic input noise augmentation}
This technique aims to reduce the model's sensitivity to input perturbations by regularizing $||\nabla_x \mathcal{L}(F(x;\theta_i), y)||_2^2.$
Since direct optimization of $||\nabla_x \mathcal{L}(F(x;\theta_i), y)||_2^2$ involves second-order gradients that are computationally prohibitive in high-dimensional spaces \cite{convex_opt}, we develop a practical approximation solution as follows.
Considering Gaussian noise injection $\upsilon \sim \mathcal{N}(0, \sigma^2 I)$ to input samples, we analyze the first-order Taylor expansion of the perturbed loss function:
$$
\mathcal{L}(F(x + \upsilon;\theta_i), y) = \mathcal{L}(F(x;\theta_i), y) + \upsilon^T \nabla_x \mathcal{L}(F(x;\theta_i), y),
$$
where higher-order terms are neglected under small $\sigma$.
Rearranging the above equation and multiplying both sides by $\upsilon$:
$$
\upsilon \upsilon^T \nabla_x \mathcal{L}(F(x;\theta_i), y)  = \upsilon [ \mathcal{L}(F(x + \upsilon;\theta_i), y) - \mathcal{L}(F(x;\theta_i), y) ].
$$
Taking the expectation over $\upsilon$ and leveraging the identity $\mathbb{E}_{\upsilon} [\upsilon \upsilon^T] = \sigma^2 I$, we have:
\begin{equation}
\begin{split}
\nonumber
    &\sigma^2 \nabla_x \mathcal{L}(F(x;\theta_i), y) = \mathbb{E}_{\upsilon} [ \upsilon \upsilon^T \nabla_x \mathcal{L}(F(x;\theta_i), y) ] \\
    = &\mathbb{E}_{\upsilon} \{ \upsilon [ \mathcal{L}(F(x + \upsilon;\theta_i), y) - \mathcal{L}(F(x;\theta_i), y) ] \}.
\end{split}
\end{equation}
For practical implementation, we employ a single-sample Monte Carlo approximation:
\begin{equation}
\begin{split}
\nonumber
    &\nabla_x \mathcal{L}(F(x;\theta_i), y) =  \frac{ \mathcal{L}(F(x + \upsilon;\theta_i), y) - \mathcal{L}(F(x;\theta_i), y) }{\sigma^2}  \upsilon.
\end{split}
\end{equation}
Then, taking the norm on both sides yields the computationally tractable regularization term:
\begin{equation}
\begin{split}
\nonumber
    &||\nabla_x \mathcal{L}(F(x;\theta_i), y)||_2 \\
    =  &\frac{| \mathcal{L}(F(x + \upsilon;\theta_i), y) - \mathcal{L}(F(x;\theta_i), y) |}{\sigma^2} || \upsilon||_2.
\end{split}
\end{equation}

\paragraph{Input-scaled trace regularization}
Input-scaled trace regularization aims to increase $\lambda_{\text{min}}$.
Here we consider $\lambda_{\text{min}}$ of classification head, since feature extraction layers are commonly shared across clients in many PFL methods.
Another consideration is that the Hessian matrix of classification head admits a closed-form solution (detailed below), facilitating more efficient computation compared to the intractable complexity required for full-network Hessian evaluations.
Let $z \in \mathbb{R}$ denote the feature embedding from feature extraction layers and let $W \in \mathbb{R}^{K \times d}$ represent the weights of classification head, such that the model output is given by $F(x;\theta_i)=W^T z$.
For cross-entropy loss, we have:
$$
\mathcal{L}(F(x;\theta_i), y) = - \sum_{k=1}^K \mathbb{I}(y=k) \log p_k. 
$$
where $p_k = \frac{\exp (w_k ^ T z)}{  \sum_{j=1}^K \exp (w_j ^ T z) }$ and $w_k$ is $k$-th column of $W$.
Then, the Hessian $H$ with respect to $W$ is given by:
$$
H(m, n) = \frac{ \partial \mathcal{L} }{\partial w_m \partial w_n } = z z^T p_m (\mathbb{I}(m=n) - p_n ).
$$
Although explicit matrix construction is feasible, spectral decomposition becomes computationally intensive for large $d$ (e.g., when $d=1024, K=10$, the matrix size is about $10^8$ with 10240 eigenvalues).
Thus, we exploit a basic algebraic relationship $tr(H)=\sum_{i=1} \lambda_i$, where $\{ \lambda_i \}$ are Hessian eigenvalues.
For classification head, this trace reduces to:
\begin{equation}
\label{trace_regular}
tr(H) = \underbrace{||z||_2^2}_{\text{feature scaling}} \cdot \underbrace{\sum_{i=1}^K p_i (1-p_i)}_{\text{prediction uncertainty}}.
\end{equation}
As can be seen, the prediction uncertainty term encourages confident predictions by penalizing ambiguous distributions, while the feature scaling term regulates the model's Lipschitz constant by constraining the magnitude of feature embeddings\footnote{The Lipschitz constant of a function measures the maximum rate at which the function can change with respect to small changes in its input (see the footnote of Assumption \ref{assumpt1} for its formal definition). A smaller Lipschitz constant implies that the function is less sensitive to input perturbations, thereby enhancing the model's robustness. By constraining the magnitude of feature embeddings, we can effectively limit the Lipschitz constant, making the model more stable and less susceptible to adversarial attacks.}.
Together, Equation \ref{trace_regular} indeed promotes confident and accurate predictions without significantly increasing the model's Lipschitz constant.

\paragraph{Parameter sensitivity maximization}
This technique aims to enhance the sensitivity of model parameters to input data.
We focus on the parameter sensitivity of classification head.
The second-order derivative $\nabla_{z,w}^2 \mathcal{L}$ of classification head enjoys a closed-form expression as follows:
$$
\nabla_{z,w}^2 \mathcal{L} = (p-y) \otimes I + z \otimes [ W (\text{diag}(p) - p p^T  ) ],
$$
where $p=(p_i)$, $\otimes$ denotes outer product, $I$ is identity matrix, and, $\text{diag}(\cdot)$ converts the input vector into a diagonal matrix.
To maximize parameter sensitivity, we incorporate the negative $L_2$ norm of $\nabla_{z,w}^2 \mathcal{L}$ into the loss function.

\paragraph{Fianl formulation}
Incorporating these techniques, client $i$'s complete optimization objective can be formulated:
\begin{equation}
\begin{split}
\label{eq_final_formulation}
&\min_{\theta_i} \ \mathbb{E}_{(x,y) \sim \mathcal{D}_i } \big\{ \mathcal{L}(F(x;\theta_i), y) + \beta_1 \cdot ||\nabla_x \mathcal{L}(F(x;\theta_i), y)||_2 + \\
&  \beta_2 \cdot tr(H) - \beta_3 \cdot ||\nabla_{z,w}^2 \mathcal{L}(F(x;\theta_i), y)||_2  \big\} \\
&= \mathbb{E}_{(x,y) \sim \mathcal{D}_i } \big\{ (1 - \frac{\gamma \beta_1 ||\upsilon||_2}{\sigma^2}) \mathcal{L}(F(x;\theta_i), y) \\
&+  \frac{ \gamma \beta_1 ||\upsilon||_2}{\sigma^2} \mathcal{L}(F(x+\upsilon;\theta_i), y) + \beta_2 ||z||_2^2 \cdot \sum_{i=1}^K p_i (1-p_i) \\
&- \beta_3 ||(p-y) \otimes I + z \otimes [ W (\text{diag}(p) - p p^T  )||_2 \},
\end{split}
\end{equation}
where $\beta_1, \beta_2, \beta_3$ are hyperparameters controlling the strengths of the respective regularization terms.
Moreover, $\gamma$ is a binary indicator that equals $1$ if $\mathcal{L}(F(x+\upsilon;\theta_i), y) \geq \mathcal{L}(F(x;\theta_i), y)$ and $-1$ otherwise.

\begin{table*}[!ht]
\small
\caption{The performance of the proposed robustness-enhanced framework in GTSRB.}
\label{sec5_tab3}
\begin{tabular}{@{}c|cccccccccccccccc@{}}
\toprule
PFL Method   & \multicolumn{2}{c}{FedAS} & \multicolumn{2}{c}{FedBABU} & \multicolumn{2}{c}{FedBN} & \multicolumn{2}{c}{FedCAC} & \multicolumn{2}{c}{FedProx} & \multicolumn{2}{c}{FedRep} & \multicolumn{2}{c}{GPFL} & \multicolumn{2}{c}{SCAFFOLD} \\ \midrule
Metric       & Acc         & AD          & Acc          & AD           & Acc         & AD          & Acc          & AD          & Acc          & AD           & Acc          & AD          & Acc         & AD         & Acc           & AD           \\ \midrule
Vanilla      & 99.53       & 61.80       & 98.32        & 79.00        & 98.85       & 84.75       & 97.81        & 72.99       & 98.17        & 85.59        & 99.46        & 68.29       & 98.23       & 51.42      & 97.87         & 79.28        \\
Ours          & 99.36       & 22.54       & 98.42        & 21.74        & 99.24       & 67.69       & 97.07        & 55.13       & 98.64        & 69.54        & 99.34        & 62.24       & 97.54       & 42.94      & 98.67         & 62.36        \\
Ours+AT & 97.06       & 13.55       & 97.47        & 13.86        & 96.83       & 22.34       & 91.25        & 23.29       & 95.77        & 24.73        & 96.79        & 16.15       & 90.31       & 24.16      & 97.14         & 22.01        \\
Ours+AT+RS    & 53.85       & 4.82        & 54.68        & 5.64         & 44.61       & 7.59        & 54.30        & 8.82        & 39.35        & 6.22         & 47.41        & 4.57        & 48.12       & 5.87       & 40.33         & 6.82         \\ \bottomrule
\end{tabular}
\end{table*}

\begin{table*}[]
\small
\caption{The performance of our framework against state-of-the-art transfer-based attacks in CIFAR-10.}
\label{sec5_tab4}
\begin{tabular}{@{}c|c|cccccccc@{}}
\toprule
Method                      & Attack    & FedAS & FedBABU & FedBN & FedCAC & FedProx & FedRep & GPFL  & SCAFFOLD \\ \midrule
\multirow{4}{*}{Ours+AT}    & PGD       & 8.46  & 8.42    & 17.12 & 11.03  & 17.02   & 11.50  & 14.19 & 18.11    \\
                            & MI        & 10.38 & 10.42   & 19.73 & 12.19  & 20.53   & 13.89  & 21.10 & 21.78    \\
                            & PCIFGSM   & 10.41 & 10.46   & 19.80 & 12.23  & 20.58   & 13.94  & 21.12 & 21.87    \\
                            & MaskBlock & 11.25 & 11.22   & 20.90 & 13.02  & 21.59   & 14.90  & 22.84 & 23.13    \\ \midrule
\multirow{4}{*}{Ours+AT+RS} & PGD       & 3.57  & 3.51    & 6.69  & 5.94   & 4.91    & 4.36   & 3.40  & 5.62     \\
                            & MI        & 4.25  & 4.35    & 7.62  & 6.54   & 5.45    & 5.40   & 5.75  & 7.59     \\
                            & PCIFGSM   & 4.28  & 4.40    & 7.61  & 6.58   & 5.46    & 5.44   & 5.74  & 7.64     \\
                            & MaskBlock & 4.74  & 4.85    & 8.52  & 7.17   & 6.00    & 6.03   & 6.74  & 8.56     \\ \bottomrule
\end{tabular}
\end{table*}

\begin{table*}[!ht]
\small
\caption{The time (s) required to execute a training round on the CIFAR-10.}
\label{sec5_overhead}
\begin{tabular}{@{}ccccccccc@{}}
\toprule
Method  & FedAS  & FedBABU & FedBN   & FedCAC & FedProx & FedRep & GPFL    & SCAFFOLD \\ \midrule
Vanilla & 253.27 & 13.14   & 13.68   & 14.06  & 15.43   & 14.40  & 23.87   & 15.74    \\
Ours    & 255.99 & 18.69   & 16.85   & 19.52  & 17.50   & 17.72  & 27.98   & 17.31    \\
Ours+AT      & 514.16      & 244.476 & 257.478 & 253.36 & 235.656 & 157.696 & 273.914 & 262.042  \\ \bottomrule
\end{tabular}
\end{table*}

\subsection{Evaluation}

\paragraph{Setup}
The experimental configurations follow the specifications outlined in Section \ref{sec_3_2}.
For our framework, we set $\beta_1=\beta_2=\beta_3=0.01, \sigma^2=1$.
To comprehensively evaluate the performance of PFL methods, we include an additional metric: clean accuracy (Acc), which is the accuracy of clients' personalized models on their respective unperturbed datasets, reflecting their generalization performance under normal conditions.
Moreover, we increase $K$ to 20 to thoroughly validate the robustness of our framework.
See Appendix \ref{appendix_supp_exp} for ablation study (the proposed three techniques).

\paragraph{The effectiveness of the proposed framework}
We first benchmark the proposed framework against vanilla implementations of various PFL methods, with the results summarized in Tables \ref{sec5_tab1}$\sim$\ref{sec5_tab3}.
As can be seen, our method consistently reduces AD across all PFL methods compared to their vanilla counterparts.
For instance, in CIFAR-10, FedBABU's AD decreases from 57.70\% to 30.72\%, representing a reduction of approximately 30\%.
Similar improvements are observed in other PFL methods, such as FedBN (9.94\%) and FedCAC (8.03\%) in CIFAR-10.
These results confirm our framework's effectiveness in countering transfer-based attacks.

Moreover, surprisingly, our framework also leads to consistent improvements in Acc for most PFL methods.
Particularly, in FedBABU, our framework boosts Acc from 66.72\% to 78.94\% (+12.2\%), a notable gain.
This simultaneous improvement in both robustness and Acc is noteworthy because there is typically a trade-off between performance and robustness.
Generally, enhancing robustness often comes at the cost of reduced Acc (see below).
In contrast, our method achieves a win-win scenario where both metrics improve concurrently.

\paragraph{Integration with traditional adversarial defenses.}
To explore complementary robustness gains, we integrate our framework with two widely-used traditional adversarial defenses: adversarial training (AT) \cite{PGD,adv_training1} and random smoothing (RS) \cite{RS}.
Both defenses can bolster the robustness of models but also introduce trade-offs in terms of Acc.
Below, we analyze the results in detail.

AT involves augmenting clients' local training dataset with adversarial examples to explicitly teach the model to resist adversarial attacks.
We implement AT by replacing $x$ in Equation \ref{eq_final_formulation} with $x+\delta$ where $\delta$ is generated via 20-step PGD attacks ($\epsilon=\frac{8}{255}$).
As Tables \ref{sec5_tab1}$\sim$\ref{sec5_tab3} illustrate, when combined with AT, our framework yields superior AD reduction across all PFL methods.
In CIFAR-10, FedBABU's AD further decreases from 30.72\% to 8.42\%, albeit with a 8\% Acc reduction.
This aligns with expectations that explicit adversarial exposure enhances robustness at generalization costs.

RS involves adding random noise to the input data during inference to mitigate the impact of adversarial perturbations.
We implement RS through Gaussian noise injection (with a variance of 1) followed by majority voting over 50 repeated inferences.
The results in tables \ref{sec5_tab1}$\sim$\ref{sec5_tab3}, while RS integration provides moderate AD reduction, its impact is less pronounced than AT's.
This trade-off highlights that while both AT and RS can greatly enhance robustness, they often sacrifices generalization performance on clean data.
Nonetheless, the combination of our framework with AT and RS represents a powerful approach for scenarios where robustness is prioritized over accuracy.

\paragraph{The robustness against state-of-the-art transfer-based attacks}
Table \ref{sec5_tab4} reports the evaluation results of our framework's robustness against state-of-the-art transfer-based attacks. 
We observe that these attacks achieve higher AD, but the increase is limited to within 5\%.
In comparison, as shown in Figure \ref{sec3_fig4}, the AD for vanilla implementations generally hovers around 50\%.
Thus, our framework indeed significantly enhances the resistance of existing PFL methods against transfer-based attacks.

\paragraph{Overhead analysis}
Here, we analyze the training and inference overhead.
We use a single NVIDIA GTX 4090 GPU.
As can be seen in Table \ref{sec5_overhead}, our framework introduces a small extra time overhead compared to the vanilla PFL methods, requiring approximately 3$\sim$5 seconds per training round.
Actually, the primary time overhead in both training and inference is dominated by the forward and backward passes, and our framework does not require any extra forward and backward passes.
This stands in sharp contrast to AT which incurs substantial training overhead by up to $(K+1) \times$ compared to standard training.
This increase occurs because PGD involves $K$ forward and backward passes to generate adversarial examples, in addition to a forward and backward pass to compute the model parameter gradients.
Regarding inference, the models for Vanilla, Ours, and AT are identical, so their inference costs are almost the same.
RS requires processing multiple noise-perturbed samples to conduct majority voting, potentially prolonging inference overhead by the noise sampling factor.
Empirically, the inference time for a batch size of 128 is approximately 0.001 seconds.
For RS, the inference time is around 0.05 seconds.
Practitioners can choose the solution that best fits their needs based on robustness, performance, and time costs.

%% file: tex/conclusion.tex
This work presented a systematic exploration of adversarial vulnerability in PFL systems, revealing critical security gaps in distributed edge systems.
The proposed theoretical framework identifies five key determinants of attack transferability—model sensitivity, loss landscape geometry, distribution alignment, fitness difference, and parameter sensitivity, providing interpretable metrics for assessing PFL robustness.
To mitigate these risks, we introduced a lightweight yet effective defense framework incorporating stochastic input noise augmentation, input-scaled trace regularization, and parameter sensitivity minimization.
Evaluations show our framework can considerably improve the robustness of PFL methods without compromising their performance.
\textbf{\textit{The source code will be made publicly available upon the acceptance of this paper to facilitate reproducibility.}}

%% file: tex/appendix.tex
\section{Proof of Theorem \ref{thm1}}
\label{theory_proof}

\begin{proof}

\textbf{Step 1: Gradient Difference Analysis.}
Let $\Delta=\theta_i-\theta_j$.
For gradient alignment, consider the following input gradient difference:
\begin{equation}
\begin{split}
\nonumber
&\nabla_x \mathcal{L}(F(x;\theta_i), y) - \nabla_x \mathcal{L}(F(x;\theta_j), y) \\
= &\nabla_x \mathcal{L}(F(x;\theta_j + \Delta), y) - \nabla_x \mathcal{L}(F(x;\theta_j), y) \\
= &\nabla_x \mathcal{L}(F(x;\theta_j), y) + \nabla_{x,\theta_j}^2 \mathcal{L}(F(x;\theta_j), y)^T \Delta \\
- &\nabla_x \mathcal{L}(F(x;\theta_j), y) + \mathcal{O}(||\Delta||_2) \\
= &\nabla_{x,\theta_j}^2 \mathcal{L}(F(x;\theta_j), y)^T \Delta + \mathcal{O}(||\Delta||_2).
\end{split}
\end{equation}
Squaring both sides and applying Cauchy-Schwarz:
\begin{equation}
\begin{split}
\nonumber
&||\nabla_x \mathcal{L}(F(x;\theta_i), y)||_2^2 + ||\nabla_x \mathcal{L}(F(x;\theta_j), y)||_2^2 \\
&- 2 \nabla_x \mathcal{L}(F(x;\theta_i), y)^\top \nabla_x \mathcal{L}(F(x;\theta_j), y) + \mathcal{O}(||\Delta||_2) \\
&\leq || \nabla_{x,\theta_j}^2 \mathcal{L}(F(x;\theta_j), y)||_2^2 \ || \Delta ||_2^2.
\end{split}
\end{equation}
Rearranging the above equation yields:
\begin{equation}
\begin{split}
\label{appendix_eq_1}
&||\nabla_x \mathcal{L}(F(x;\theta_i), y)||_2^2 + ||\nabla_x \mathcal{L}(F(x;\theta_j), y)||_2^2 \\
&- || \nabla_{x,\theta_j}^2 \mathcal{L}(F(x;\theta_j), y)||_2^2 || \Delta ||_2^2 + \mathcal{O}(||\Delta||_2) \\
&\leq  2 \nabla_x \mathcal{L}(F(x;\theta_i), y)^\top \nabla_x \mathcal{L}(F(x;\theta_j), y).
\end{split}
\end{equation}
We next derive the bound of $||\Delta||_2^2$.

\textbf{Step 2: Bounding $\Delta$.}
Using the Kantorovich-Rubinstein duality, the Wasserstein distance between distributions admits the dual form: $K \cdot W(P_{XY}^{(i)}, P_{XY}^{(j)}) =  \text{sup}_{||f||_L \leq K} \mathbb{E}_{P_{XY}^{(i)}}[f(x,y)] - \mathbb{E}_{P_{XY}^{(j)}}[f(x,y)]$, where $||f||_L \leq K$ denotes K-Lipschitz continuity.
For the model-specific loss function, this implies: $K \cdot W(P_{XY}^{(i)}, P_{XY}^{(j)}) \leq \mathbb{E}_{P_{XY}^{(i)}}[\mathcal{L}(F(x; \theta_i), y)] - \mathbb{E}_{P_{XY}^{(j)} }[\mathcal{L}(F(x; \theta_i), y)]$.
We expand the loss difference through Taylor approximation:
\begin{equation}
\begin{split}
\nonumber
&\quad K \cdot W(P_{XY}^{(i)}, P_{XY}^{(j)}) - \mathbb{E}_{P_{XY}^{(i)}} \left[ \mathcal{L}(F(x;\theta_i), y) \right] \\
&\leq - \mathbb{E}_{P_{XY}^{(j)}} \left[ \mathcal{L}(F(x;\theta_i), y) \right] = - \mathbb{E}_{P_{XY}^{(j)}} \left[ \mathcal{L}(F(x;\theta_j + \Delta), y) \right] \\
&= - \mathbb{E}_{P_{XY}^{(j)}} \bigg[ \mathcal{L}(F(x;\theta_j), y) + \nabla_{\theta_j} \mathcal{L}(F(x;\theta_j), y)^\top \Delta \\
&+ \Delta^{\top} \nabla^2_{\theta_j} \mathcal{L}(F(x;\theta_j), y) \Delta + \mathcal{O}(||\Delta||_2^2) \bigg].
\end{split}
\end{equation}
Using local optimality of $\theta_j$ on $P_{XY}^{(j)}$ gives:
\begin{equation}
\begin{split}
\nonumber
     &- \mathbb{E}_{P_{XY}^{(i)}} \left[ \mathcal{L}(F(x;\theta_i), y) \right] + \mathbb{E}_{P_{XY}^{(j)}} \left[ \mathcal{L}(F(x;\theta_j), y) \right] \\
     &+ K \cdot W(P_{XY}^{(i)}, P_{XY}^{(j)}) + \mathcal{O}(||\Delta||_2^2) \\
     &\leq - \mathbb{E}_{P_{XY}^{(j)}} \bigg[  \Delta^{\top} \nabla^2_{\theta_j} \mathcal{L}(F(x;\theta_j), y) \Delta \bigg].
\end{split}
\end{equation}

Under strong convexity, the quadratic term satisfies \\ $-\Delta^{\top} \nabla^2_{\theta_j} \mathcal{L}(F(x;\theta_j), y) \Delta \leq -\lambda_{\min} ||\Delta||_2^2.$
Thus, we have:
\begin{equation}
\begin{split}
\label{appendix_eq_2}
     & - \mathbb{E}_{P_{XY}^{(i)}} \left[ \mathcal{L}(F(x;\theta_i), y) \right] + \mathbb{E}_{P_{XY}^{(j)}} \left[ \mathcal{L}(F(x;\theta_j), y) \right] \\
     &+ K W(P_{XY}^{(i)}, P_{XY}^{(j)}) +  + \mathcal{O}(||\Delta||_2^2) 
    \leq - \lambda_{\text{min}} ||\Delta||_2^2.
\end{split}
\end{equation}

\textbf{Step 3: Combining Step 1 and Step 2.}
Substituting Equation \ref{appendix_eq_1} into Equation \ref{appendix_eq_2} through $||\Delta||_2^2$ yields:
\begin{equation}
\begin{split}
& \mathcal{O}(||\Delta||) + ||\nabla_x \mathcal{L}(F(x;\theta_i), y)||_2^2 + ||\nabla_x \mathcal{L}(F(x;\theta_j), y)||_2^2 \\
&- \frac{|| \nabla_{x,\theta_j}^2 \mathcal{L}(F(x;\theta_j), y)||_2^2}{\lambda_{\text{min}}}  (-\mathbb{E}_{P_{XY}^{(i)}} \left[ \mathcal{L}(F(x;\theta_i), y) \right] \\
&+ \mathbb{E}_{P_{XY}^{(j)}} \left[ \mathcal{L}(F(x;\theta_j), y) \right] + K \cdot W(P_{XY}^{(i)}, P_{XY}^{(j)})) \\
&\leq  2 \nabla_x \mathcal{L}(F(x;\theta_i), y)^\top \nabla_x \mathcal{L}(F(x;\theta_j), y).
\end{split}
\end{equation}
This completes the proof of the gradient alignment bound.

\end{proof}

\section{Supplementary Experiment}
\label{appendix_supp_exp}

\begin{table}[!t]
\caption{Ablation study on the impact of key components on adversarial robustness, measured by AD. Lower AD indicates better robustness. AD is averaged over eight PFL methods.}
\label{tab:ablation_study}
\small
\begin{tabular}{@{}cc@{}}
\toprule
Component                                & AD    \\ \midrule
w.o. Stochastic input noise augmentation & 44.04 \\
w.o. Input-scaled trace regularization   & 41.65 \\
w.o. Parameter sensitivity maximization  & 39.92 \\
All                                      & 38.73 \\ \bottomrule
\end{tabular}
\end{table}

\begin{table*}[!t]
\caption{Impact of model architecture on robustness across various PFL methods.}
\label{tab:model_impact}
\small
\begin{tabular}{ccccccccc}
\toprule
Model    & FedAS & FedBABU & FedBN & FedCAC & FedProx & FedRep & GPFL  & SCAFFOLD \\ \midrule
ResNet10 & 47.61 & 57.79   & 62.94 & 36.19  & 55.14   & 50.38  & 23.52 & 60.34    \\
ResNet18 & 55.22 & 66.57   & 70.55 & 42.42  & 62.33   & 61.69  & 34.60 & 73.90    \\
ResNet34 & 58.85 & 70.59   & 72.41 & 45.80  & 66.06   & 65.22  & 38.19 & 77.29    \\ \bottomrule
\end{tabular}
\end{table*}

\paragraph{The robustness of different models}
Table \ref{tab:model_impact} evaluates the influence of model architecture on the robustness across various PFL methods.
Notably, deeper architectures (e.g., ResNet34) generally exhibit higher AD compared to shallower ones (e.g., ResNet10), suggesting that increased model capacity may lead to greater vulnerability to adversarial attacks.

\paragraph{Ablation study on key components}
Table \ref{tab:ablation_study} presents an ablation study analyzing the impact of removing individual components from our framework.
Each row corresponds to a specific component, and AD quantifies the robustness when that component is excluded.
The results indicate that all components contribute to reducing AD, with the full components achieving the lowest of 38.73.
This suggests that each component plays a complementary role in improving robustness.
Moreover, parameter sensitivity maximization term has the smallest impact, and stochastic input noise augmentation has the largest impact.
